\documentclass{article}
\usepackage[utf8]{inputenc}

\title{Noether: The More Things Change, the More Stay the Same}
\author{Grzegorz Głuch\footnote{Correspondence to Grzegorz Głuch: grzegorz.gluch@epfl.ch}\\EPFL \and Rüdiger Urbanke\\EPFL}
\date{}

\usepackage{natbib}
\usepackage{graphicx}
\usepackage{xcolor}

\usepackage{subfig}

\usepackage[margin=1in]{geometry}
\usepackage{amssymb,amsmath,amsthm}
\usepackage{verbatim,hyperref}
\usepackage{dsfont}
\usepackage{algorithm}
\usepackage[noend]{algpseudocode}
\usepackage{thmtools,thm-restate}

\usepackage{mathtools}
\usepackage{enumerate}
\usepackage{dirtytalk}

\usepackage[parfill]{parskip} % Activate to begin paragraphs with an empty line rather than an indent
\usepackage{marginnote} % for better margin notes
\usepackage{algorithm}
\usepackage[noend]{algpseudocode}
\usepackage{booktabs}
\usepackage{siunitx}
\newcommand{\N}{\mathbb{N}}

\newcommand{\R}{\mathbb{R}}

\newtheorem{lemma}{Lemma}
\newtheorem{observation}{Observation}

\newtheorem{remark}{Remark}

\theoremstyle{definition}
\newtheorem{definition}{Definition}
\newtheorem{example}{Example}

\newcommand{\e}{\epsilon}

\newcommand{\x}{\pmb{x}} % generic vector of functions when we talk about Lagrangian 
\newcommand{\xdot}{\dot{\pmb{x}}} % time derivative of \x
\newcommand{\xddot}{\ddot{\pmb{x}}} %second time derivative of \x
\newcommand{\dx}{\pmb{\xi}} % time derivative of \x
\newcommand{\dxdot}{\dot{\pmb{\xi}}} %second time derivative of \x

\newcommand{\w}{\pmb{w}} % generic vector of functions when we talk about Lagrangian 
\newcommand{\wdot}{\dot{\pmb{w}}} % time derivative of \x
\newcommand{\wddot}{\ddot{\pmb{w}}} %second time derivative of \x
\newcommand{\y}{\pmb{y}} 
\newcommand{\uu}{\pmb{u}} 
\newcommand{\alphap}{\pmb{\alpha}} 

\newcommand{\bet}{\pmb{\beta}} %vector of parameters for Swish
\newcommand{\Lag}{{\mathcal L}} % Lagrangian
\newcommand{\pot}{U} % Lagrangian
\newcommand{\Act}{J} % the action associated to a Lagrangian
\newcommand{\vd}{\nabla \Act} % variational derivative

\newcommand{\loss}{{L}} % loss function

\newcommand{\trace}{{\text{tr}}} % trace
\newcommand{\pred}{{f}}

\newcommand{\cnstdd}{{\kappa_2}}
\newcommand{\cnstd}{{\kappa_1}}

\usepackage{todonotes}
\newcounter{mycomment}

\usepackage{tabularx}
\usepackage{array,collcell}
\newcommand\AddLabel[1]{%
  \refstepcounter{equation}% increment equation counter
  (\theequation)% print equation number
  \label{#1}% give the equation a \label
}
\newcolumntype{M}{>{\hfil$\displaystyle}X<{$\hfil}} % mathematics column
\newcolumntype{L}{>{\collectcell\AddLabel}r<{\endcollectcell}}

\begin{document}

\maketitle

\begin{abstract}
{\em Symmetries} have proven to be important ingredients in the analysis of neural networks. So far their use has mostly been implicit or seemingly coincidental.

We undertake a systematic study of the role that symmetry plays. In particular, we clarify how symmetry {\em interacts} with the learning algorithm. The key ingredient in our study is played by Noether's celebrated theorem which, informally speaking, states that symmetry leads to {\em conserved quantities} (e.g., conservation of energy or conservation of momentum). In the realm of neural networks under gradient descent, model symmetries imply restrictions on the gradient path. E.g., we show that symmetry of activation functions leads to boundedness of weight matrices, for the specific case of linear activations it leads to balance equations of consecutive layers, data augmentation leads to gradient paths that have \say{momentum}-type restrictions, and time symmetry leads to a version of the Neural Tangent Kernel.  

Symmetry alone does not specify the optimization path, but the more symmetries are contained in the model the more restrictions are imposed on the path. Since symmetry also implies over-parametrization, this in effect implies that some part of this over-parametrization is cancelled out by the existence of the conserved quantities. 

Symmetry can therefore be thought of as one further important tool in understanding the performance of neural networks under gradient descent.
\end{abstract}

%\tableofcontents

\section{Introduction}
There is a large body of work dedicated to understanding what makes neural networks (NNs) perform so well under versions of gradient descent (GD). In particular, why do they generalize without overfitting despite their significant over-parametrization. Many pieces of the puzzle have been addressed so far in the literature.  Let us give a few examples.

One recent idea is to analyze NNs via their kernel approximation. This approximation appears in the limit of infinitely-wide networks. This line of research was started in \citet{jackotntk}, where the authors introduced the Neural Tangent Kernel (NTK). Using this approach it is possible to show convergence results, prove optimality and give generalization guarantees in some learning regimes (e.g. \citet{du2018provably, chizat19lazy, zou2018stochastic, allen2018learning, arora2019fine}). These results beg the question whether the current success of NNs can be entirely understood via the theory that emerges when the network width tends to infinity. This does not seem to be the case.
E.g., it was shown in \citet{yehudai19power}, \citet{allenzhu19resnets, allenzhu20backward}, \citet{daniely20parities}, \citet{montanari1, montanari2} that there are examples where GD provably outperforms NTK and, more generally, any kernel method. \say{Outperform} here means that GD provably has a smaller loss than NTK. In \citet{srebroabbe} the authors give an even stronger result, by providing examples where NTK does not improve on random guessing but GD can learn the problem to any desired accuracy. Hence, there is more to NNs under GD than meets the eye of NTKs. 

Another idea for analyzing NNs is to use the so-called mean field method, see e.g. \citet{montanariMeanField, montanaridimensionfree, meanfieldRotskoff}. Important technical objects in this line of work are often the Wasserstein space of probability distributions and the gradient flow on this space. Similar tools were used in \citet{bachoptimaltransport} to show that for 2-layer NNs in a properly chosen limit if GD converges it converges to the global optimum.

All of the above mentioned papers use {\em trajectory-based} approaches. I.e., one analyzes a trajectory of a specific optimization algorithm. The approach we take in this paper can be classified as such. There is an alternative approach that tries to characterize geometric properties of the whole optimization landscape (see e.g. \citet{Haeffele2015, ShamirSpurious, Freeman2017TopologyAG, zhou2018critical, Maidentity, pmlr-v70-nguyen17a}). If the landscape does not contain local minima and all saddle points are strict then it is possible to guarantee convergence of GD (\citet{jordanGDdoesntconverge, jordanEscapeSaddle}). Unfortunately these properties don't hold even for shallow networks (\citet{yun2018global}). 

One of the perhaps oldest pieces of \say{wisdom} in ML is that the bias–variance trade-off curve has a \say{U}-shape -- you get a large generalization error for very small model capacities, the error then decays when you increase the capacity, but the error eventually rises again due to overfitting.
In \cite{Belkin15849} the authors argue that this \say{U}-shaped curve is in fact only the first part of a larger curve that has a \say{double descent} shape and that NNs under stochastic gradient descent (SGD) operate on the \say{new} part of this curve at the very right where the generalization error decays again. Indeed, it was well known that NNs have a very large capacity and that they can even fit random data, see \citet{zhang2016understanding}.
A closely related phenomenon was found earlier by \citet{Spiglerjamming}, who used a concept called the “jamming transition,” to study the transition from the {\em under-parametrized} to the {\em over-parametrized} regime. There is a considerable literature that has confirmed the observation by \cite{Belkin15849}, see e.g., \cite{hastie2020surprises} and the many references therein.

\begin{comment}
NTK: \cite{jackotntk}- first NTK; crucial properties: I guess i.i.d. of the weight distribution and permutational symmetry. \cite{chizatimplicitbias} - maybe relevant, it's a follow up of \cite{bachoptimaltransport}, now reading a blog \href{https://francisbach.com/gradient-descent-for-wide-two-layer-neural-networks-implicit-bias/}{https://francisbach.com/gradient-descent-for-wide-two-layer-neural-networks-implicit-bias/} they have a lot of references for infinite-wide NNs, crucial properties: homogeneity again. \citet{montanaridimensionfree} - explicit bound for the number of hidden neurons needed for convergence. \cite{malach2021quantifying} separation results for NTK

mean-field model: \cite{montanariMeanField} - mean field Andrea Montanari, \cite{bachoptimaltransport} - if GD converges it convergence to global optimum (for 2-layer networks in infinite width regime); crucial properties used: homogeneity, separability (no weight sharing, the prediction function is a sum of independent components), at initialization weights are uniformly distributed on sphere (I think something similar was needed in \cite{montanariMeanField}); tools: optimal transport

over-parametrization and (S)GD: \cite{NEURIPS2019_05e97c20}
\end{comment}

There has also been a considerable literature on improving generalization bounds. Since this will not play a role in our context we limit our discussion to providing a small number of references, see e.g., \cite{pmlr-v80-arora18b,DBLP:journals/corr/NeyshaburBMS17,DBLP:journals/corr/NeyshaburBMS17aa,NIPS2017_b22b257a,DBLP:journals/corr/NeyshaburTS15, NEURIPS2019_05e97c20}.

\paragraph{Symmetries.} One further theme that appears frequently in the literature is that of {\em symmetry}. In particular, we are interested in how a chosen optimization algorithm interacts with the symmetries that are present in the model.
Indeed, our curiosity was piqued by the conserved quantities that appear in the series of papers \citet{aroralinearnetworks, arora2019implicit, weihubalance} and we wanted to know where they came from.
Let us start by discussing how symmetry is connected to over-parametrization before discussing our contributions.

\begin{comment} 
As discussed above modern ML systems operate in the over-parametrized regime, that is where the number of parameters of a model is larger than the number of data samples it has access to during training. Traditional statistical learning theory suggests that increasing the model capacity results in over-fitting but this is not what is observed when training large NNs under stochastic GD. 
%To the contrary, it has been observed that increasing model complexity in fact improves generalization
Standard complexity measures like the VC dimension cannot capture this behavior as they grow with the network size. As a result, many other measures have been proposed to date to tackle this problem. These measures are expressed using objects like norms, sharpness or margins. Using these new measures some progress has been made in explaining the observed behavior but no complete theory exists to date.
\end{comment}

One can look at over-parametrization through the following lens. Let us assume that the network is specified by a vector, i.e., we think of the parameters as a vector of scalars (the weights and biases). In the sequel we simply refer to this vector as the parameters. The parameters lie in an allowable space. This is the space the learning algorithms operates in. Due to the over-parametrization the training set does not fully specify the parameters that minimize the chosen loss. 
%In fact, there are exponentially many such parameters that correspond to minimizing (or approximately minimizing) the objective. 
%Among these minima there are some that generalize well and some that generalize poorly. 
%The question is why standard optimization algorithms such as GD or stochastic GD do seem to find minima that generalize well. 
More broadly, imagine that the space of parameters is partitioned into groups, where each group corresponds to the set of parameters that result in the same loss. Then, trivially, the loss function is {\em invariant} if we move from one element of a given group to another one. 

The groups might have a very complicated structure in general. But the structure is simple if the over-parametrization is due to {\em symmetry}. This is the case we will be interested in. More precisely, assume that the function expressed by the model itself is over-parametrized. For instance, a two-layer NN with linear activation functions computes a linear transformation of the input. It is therefore over-parametrized since we are using two matrices instead of a single one. We can thus partition the parameter space into groups according to this function. In the aforementioned case of a two-layer linear network it can be shown that
these groups contain an easily definable symmetric structure and this is frequently the case. 
As we will discuss in the detail later one, there is a set of continuous transformations of the parameters under which the prediction function does not change. 

How does a particular learning algorithm interact with this symmetric structure? A standard optimization algorithm starts at a point in the parameter space and updates the parameters in steps. When updating the parameters the algorithm moves to a new group of parameters for a new loss value. Intuitively, the bigger the group that corresponds to symmetries the harder it would appear to be to control the behavior of the learning algorithm. E.g., going back to 
the case of a two-layer NN with linear activation functions, even if the resulting linear transformation is bounded, the weights of the two component matrices might tend to zero or infinity.
We show that if the over-parametrization is due to symmetries then these symmetries lead to conserved quantities, effectively cancelling out the extra degrees of freedom -- hence the title. 

To analyze the connection between symmetries and learning algorithms we resort to a beautiful idea from physics, namely Noether's Theorem. Symmetries are ubiquitous in physics and can be considered the bedrocks on top of which physics is built. A symmetry in physics is a feature of the system that is preserved or remains unchanged under a transformation. Symmetries include translations of time or space and rotations. 
It was Emmy Noether, a German mathematician, who established in a systematic fashion how symmetries give rise to conservation laws. Informally, her famous theorem states that to every continuous symmetry corresponds a conserved quantity. This theorem proves that for instance time invariance implies conservation of energy, spatial translation invariance implies conservation of momentum and rotational invariance implies conservation of angular momentum. On a historical note, Noether's interest in this topic was peaked by a foundational question concerning the conservation of energy in the framework of Einstein's general relativity, where the symmetry/invariance is due 
to the invariance wrt to the reference frame, see \cite{rowe2019emmy}.

\paragraph{Our contribution.} 
We consider learning through the lense of Noether's Theorem -- symmetries imply conserved quantities.
This (i) allows us to unify previous results in the field and (ii) makes it clear how to obtain new such conserved quantities in a systematic fashion. In particular, we discuss three distinct ways of how symmetries emerge in the context of learning. These are (a) symmetries due to activation functions, with the special case of linear activation functions investigated separately, (b) symmetries due to data augmentation, and (c) symmetries due to the time invariance of the optimization algorithm. Let us discuss these points in more detail.

From a practical perspective, vanishing or exploding gradients are a fundamental problem when training deep NN. And from a theoretical point of view it was argued in \citet{shamirhowtoprove} that the most important barrier for proving algorithmic results is that the parameters are possibly unbounded during optimization. Therefore, any technique that can either help to keep the parameters bounded or guarantee the boundedness of the parameters a priori is of interest. 

In \citet{weihubalance} it was shown that if the network uses ReLU activation functions and we use GD then the norms of layers are balanced during training. We show how this result is a natural consequence of our general framework and discuss how it can be extended. In particular, we derive balance equations for other activation functions such as polynomial or Swish. Our framework exposes why exploding/vanishing gradients might be an inherent problem for polynomial activation functions. Finally, answering a question posed in \citet{weihubalance}, we derive balance equations for a wide class of learning algorithms, including Nesterov's Accelerated Gradient Descent.

It was shown in \citet{aroralinearnetworks} that GD converges to a global optimum for deep {\em linear} networks. This result relies crucially on a notion of balancedness of weight matrices. We show how this balance condition becomes a conserved quantity in our framework. We also show how to control the evolution of this balance condition for other learning algorithms, potentially paving the way for proving convergence for other optimizers. 

Data augmentation is a very popular method for increasing the amount of available training data. We show how data augmentation naturally leads to symmetries and we derive a corresponding conserved quantity. This quantity in turn will constrain the evolution of the optimization path in a particular way.

Symmetries play a crucial role in many of the works mentioned in the beginning of this introduction. We show how these symmetries can be seen in a unified way in our framework. For instance, we prove a key component of the Neural Tangent Kernel derivation. Our result is on the one hand weaker than the standard NTK result but the proof follows automatically from our framework and holds for more general distributions than the standard version. In \citet{montanariMeanField} it was shown that symmetries in the data distribution lead to symmetries in the weights of the network. This relates to our analysis of symmetries arising from data augmentation. In \citet{bachoptimaltransport} it was shown that if gradient descent converges it converges to the global optimum for 2-layer NN in the infinite width limit. This result crucially relied on the homogeneity symmetry that we analyze in detail. 

\section{Symmetry and Conservation Laws - Noether's Theorem}\label{sec:noether}

\subsection{The Lagrangian and the Euler-Lagrange Equations}
A fundamental idea in physics is to define the behavior of a system via its {\em Lagrangian}. 

\begin{example}[Mechanics]
Perhaps the best-known example is the Lagrangian formulation of classical mechanics,
\begin{equation} \label{equ:mechaniclagrangian}
\Lag(t, \x, \xdot)  = \frac12 m \|\xdot(t)\|^2 -\pot(\x(t)).
\end{equation}
Here, $\pmb{x}(t)$ denotes the position of a particle at time $t$. If we imagine that the particle moves in $d$-dimensional real space then $\x(t) \in \R^d$. The term $\frac12 m \|\xdot(t)\|^2$ denotes the so-called {\em kinetic} energy of this particle that is presumed to have mass $m$ and is travelling at a speed $\|\xdot(t)\|$ (the $\dot{}$ denotes the derivative with respect to time and $\|\cdot\|$ denotes the euclidean norm). The so-called {\em potential} energy is given by the term $\pot(\x(t))$, where $\pot: \R^d \rightarrow \R$ denotes the potential.
\end{example}
Given a Lagrangian $\Lag(t, \x, \xdot)$, we associate to it the functional 
\begin{align} \label{equ:action}
\Act[\x] = \int \Lag(t, \x, \xdot) dt,
\end{align}
which is called the {\em action}. This functional associates a real number to each function (path) $\x$. 
Note that the definite integral (\ref{equ:action}) is typically over a fixed time interval. The appropriate range of integration will be understood from the context. 
In what follows we will assume that all functions $\x(t)$ are members of a suitable class, e.g., the class of continuously differentiable functions $C^1$, so that all mathematical operations are well defined. But we will not dwell on this. We refer the interested reader to one of the many excellent textbooks that discuss variational calculus and Noether's theorem, such as \cite{GeF63}. Our exposition favors simplicity over mathematical rigor. The reader who is already familiar with the calculus of variation and Noether's theorem can safely skip ahead to Section~\ref{sec:extensions} which discusses the particular extension of the basic method that we will use.

The idea underlying the characterization of a system in terms of its Lagrangian is that the system will behave in such a way so as to minimize the associated action. This is called the {\em stationary action} principle. E.g., a particle will move from a given  starting position to its given ending position along that path that minimizes (\ref{equ:action}) where the Lagrangian is given by (\ref{equ:mechaniclagrangian}).

Before we consider functionals let us revisit the simpler and more familiar setting of a function $f$ that depends on several variables, i.e., $f: \R^d \xrightarrow{} \R$. We are looking for an $\x \in \R^d$ that is a minimizer of $f$. We proceed as follows.
We take $\x$ and a small {\em deviation} $d \x$. We expand the function $f(\x+d\x)$ up to linear terms in $d\x$ to get\footnote{In the sequel, if we are given a function $g(t, x(t))$ that depends on one independent variable, call it $t$, and a dependent variable, call it $x(t)$, then we write $\frac{\partial g}{\partial t}$ to denote the so-called {\em partial} derivative, whereas we write $\frac{d g}{d t}$ to denote the so-called {\em total} derivative. Recall that by the chain rule we have $\frac{d g(t, x)}{d t}= \frac{\partial g(t, x)}{\partial t} + \frac{\partial g(t, x)}{\partial x} \frac{\partial x(t)}{\partial t}$. We often use the short-hand notation $\dot{g}$ or $\dot{x}$ to denote the (total) derivative with respect to the independent variable $t$. }
\begin{align*}
    f(\x + d \x) = f(\x) + \langle \nabla f, d \x \rangle + O(\|d\x\|^2),
\end{align*}
where $\nabla f = \left(\frac{\partial f}{\partial x_1}, \cdots, \frac{\partial f}{\partial x_d} \right)^T$ is the so-called {\em gradient}. If $\x$ is a minimizer of $f(\x)$ then it must be true that the linear term $\left\langle \nabla f, d \x \right\rangle$ vanishes for any deviation $d \x$. This implies that the gradient itself must be zero, i.e.,
\begin{align} \label{equ:zerogradient}
\nabla f(\x) = 0.
\end{align}
We say that a point $\x$ that fulfills the condition (\ref{equ:zerogradient}) is an {\em extremal} or a {\em stationary} point of $f$.

Let us now look at the equivalent concept for functionals.  We are given the functional (\ref{equ:action}). We are looking for an $\x(t) \in \R^d$ that is a minimizer of $\Act$. We proceed as before. We take $\x$ and a small {\em variation} $\delta x$. In the simplest setting we require $\delta \x$ to take on the value $0$ at the two boundary points, so that $\x + \delta \x$ starts and ends at the same position as $\x$. We expand the functional $\Act[\x + \delta\x]$ up to linear terms in $\delta \x$. We claim that this has the form
\begin{align} \label{equ:expansion}
\Act[\x + \delta\x] = 
    \Act[\x] + \int \left\langle \vd, \delta \x \right\rangle dt + O \left(\int \| \delta \x\|^2 \right),
\end{align}
where $\vd$ is called the {\em variational derivative}. Note that the variational derivative plays  for functionals the same role as the gradient plays for functions of several variables. Whereas the gradient $\nabla f$ is a $d$-dimensional real-valued vector, $\vd$ is a $d$-dimensional vector whose components are real-valued functions. Further, we claim that the variational derivative can be expressed as
\begin{align} \label{equ:variationalderivative}
\vd = \frac{\partial \Lag}{\partial \x}  - \frac{d}{dt} \frac{\partial \Lag}{\partial \xdot}.
\end{align}
Before we show how to derive the expansion (\ref{equ:expansion}) and the expression (\ref{equ:variationalderivative}) let us conclude the analogy. If $\x$ is a minimizer of $\Act[\x]$ then it it must be true that the linear term 
$\int \left\langle \vd, \delta \x \right\rangle dt$ in (\ref{equ:expansion}) vanishes for any variation $\delta \x$ that vanishes at the two boundaries. This implies that
\begin{align} \label{equ:elequations}
\vd = 
%\frac{\partial \Lag}{\partial \x}  - \frac{d}{dt} \frac{\partial \Lag}{\partial \xdot} = 
0.
\end{align}
The system of equations (\ref{equ:elequations}) is known as the {\em Euler-Lagrange} (EL) equations. It is the equivalent of (\ref{equ:zerogradient}). For this reason, a function $\x(t)$ that fulfills (\ref{equ:elequations}) is called extremal or stationary, justifying the name {\em stationary action} principle mentioned above. 
\begin{example}[Mechanics Contd]
For our running example, if we compute (\ref{equ:variationalderivative}) from (\ref{equ:mechaniclagrangian}) we get
\begin{align} \label{equ:elformechanics}
\vd & =- m \ddot{\pmb{x}}(t) - \nabla \pot(\pmb{x}(t)).
\end{align}
According to (\ref{equ:elequations}), we get the EL equations $\vd = -m \ddot{\pmb{x}}(t) - \nabla \pot(\pmb{x}(t)) = 0$. This is Newton's second law: the acceleration times the mass, $m\ddot{\pmb{x}}(t)$, i.e., the change in momentum per unit time, is equal to the applied force $-\nabla \pot(\pmb{x}(t))$.
\end{example}
Let us  now get back to the expansion (\ref{equ:expansion}) and the expression for the variational derivative (\ref{equ:variationalderivative}). We follow \citet{GeF63}[p. 27]. Asssume that the integration is over the interval $\left[\underline{t}, \bar{t} \right]$. Let us divide this interval into $n+1$ evenly sized segments, $\underline{t} = t_0 < t_1 < \dots < t_{n+1} = \bar{t}$, $t_i-t_{i-1} = \Delta = (\bar{t}-\underline{t})/(n+1)$for $i=1, \dots, n+1$. We can then approximate the function $\x(t)$ by a piece-wise linear function that goes through the points $\x_i=\x(t_i)$, $i=0, \dots, n+1$. This leads to an approximation of the functional of the form
\begin{equation}\label{eq:functionalapx}
\Act[\x] \sim \sum_{i=0}^{n} \Lag(t_i, \x_i, (\x_{i+1}-\x_i)/\Delta) \Delta.
\end{equation}
If we assume that the two end points of $\x$ are fixed, and hence $\x_0$ and $\x_{n+1}$ are fixed, then this is a function of the $n$ $d$-dimensional vectors $\x_1, \cdots, \x_n$. 

For simplicity of presentation assume that $x(t) \in \R$, i.e., $d=1$. In this case we are back to the familiar setting of minimizing a function $\Act$ of $n$ variables, $\Act=\Act[x_1, \dots, x_n]$. We know that in this setting we need to compute the gradient $\nabla \Act
= \left(\frac{\partial \Act}{\partial x_1}, \dots, \frac{\partial \Act}{\partial x_n} \right)$. Each $x_k$ appears in two terms in \eqref{eq:functionalapx}, corresponding to $i = k$ and $i = k-1$. Thus we get:
$$\frac{\partial \Act}{\partial x_k} = \frac{\partial \Lag}{\partial x} \left(t_k,x_k, \frac{x_{k+1} - x_k}{\Delta} \right)\Delta - \frac{\partial \Lag}{\partial \dot{x}} \left(t_k,x_k, \frac{x_{k+1} - x_k}{\Delta} \right) + \frac{\partial \Lag}{\partial \dot{x}} \left(t_{k-1},x_{k-1}, \frac{x_{k} - x_{k-1}}{\Delta} \right) \text{.}$$
Dividing both sides of the above equation by $\Delta$ we get:
\begin{equation}\label{eq:derivativeapx}
\frac{\partial \Act}{\partial x_k \Delta} = \frac{\partial \Lag}{\partial x} \left(t_k,x_k, \frac{x_{k+1} - x_k}{\Delta} \right) - \frac{1}{\Delta} \left[ \frac{\partial \Lag}{\partial \dot{x}} \left(t_k,x_k, \frac{x_{k+1} - x_k}{\Delta} \right) - \frac{\partial \Lag}{\partial \dot{x}} \left(t_{k-1},x_{k-1}, \frac{x_{k} - x_{k-1}}{\Delta} \right) \right] \text{.}
\end{equation}
The product $\partial x_k \Delta$ that appears in the denominator on the left-hand side of 
(\ref{eq:derivativeapx}) has a geometric meaning - it is an area. What does this area correspond to? Recall that we compute by how much the functional changes if we apply a small variation. Assume that this variation is zero except around the position $x_k$ where the variation takes on the value $d x_k$ for a \say{length} of $\Delta$. In other words, we imagine that the variation $\delta x$ is the zero function except for a small triangular \say{bump} of area $\partial x_k \Delta$ around the position $x_k$. (This is the equivalent to a vector $d \x$ that is zero except for component $k$ that takes on the value $d \x_k$ in the setting of minimizing a function of $d$ variables.) The right-hand side of (\ref{eq:derivativeapx}) tells us by how much the functional changes due to this variation. Taking the limit of the expression \eqref{eq:derivativeapx} when $\Delta \xrightarrow{} 0$ we get:
$$
\vd = \frac{\partial \Lag}{\partial x} - \frac{d}{dt} \frac{\partial \Lag}{\partial \dot{x}},
$$
which is exactly the advertised variational derivative \eqref{equ:variationalderivative} for the case $d=1$. If $d>1$, then $\vd$ is a $d$-dimensional vector of functions as is the variation $\delta \x$, leading to the integral of the inner product between these two quantities as shown in (\ref{equ:elequations}).
\subsection{Invariances and Noether's Theorem}
One advantage of formulating the system behavior in terms of a Lagrangian is that in this framework it is easy to see how {\em symmetries/invariances} (e.g., time or space translations) give rise to {\em conservation} laws (e.g., conservation of energy or momentum). This is the celebrated Noether's theorem. We start by looking at two specific examples.
\begin{example}[Mechanics Contd -- Time Invariance -- Conservation of Energy]
Note that the Lagrangian (\ref{equ:mechaniclagrangian}) does not explicitly depend on time $t$. I.e., it is invariant to translations in time. An explict calculation shows that
\begin{align*}
\frac{d}{dt} \left(\frac12 m \|\xdot(t)\|^2 + \pot(\x(t)) \right)
& = m \langle \xdot(t), \xddot(t) \rangle + \langle \nabla \pot(\x(t)), \xdot(t) \rangle 
=\langle m \xddot(t)+\nabla \pot(\x(t)), \xdot(t) \rangle 
 \stackrel{(\ref{equ:elformechanics})}{=} -\langle \vd, \xdot \rangle 
  = 0.
\end{align*}
Note that in the last step we have used the trivial observation that if $\vd=0$ then $\langle \vd, \xdot \rangle 
  = 0$.
In words, the quantity $\frac12 m \|\xdot(t)\|^2 + \pot(\x(t))$ stays conserved for an extremal path. But note that this conserved quantity is the sum of the kinetic and potential energy.
\end{example}

\begin{example}[Mechanics Contd -- Spatial Invariance -- Conservation of Momentum]\label{ex:spatialsymmetry}
In our running example assume that $\x \in \R^3$ ($x$, $y$, and $z$ components) and that the potential $\pot(\x)$ is constant along the $x$ and $y$ component. Assume that the function $\x(t)$ is extremal, i.e., $\x(t)$ is a solution of the EL equations $\vd=0$.

Let $\delta \x$ be any variation. Due to the extremality of $\x(t)$, $\langle \vd, \delta \x \rangle=0$. Now specialize to
$\delta \x(t) = ( dx, dy, 0)^T$ 
and note that $\pot$ is invariant in the $x$ and $y$ direction
and hence $(\nabla \pot)_x = (\nabla \pot)_y=0$. Then we get $0=-\langle \vd, \delta \x \rangle  = \langle m \ddot{\x} + \nabla \pot(\x), (dx, dy, 0)^T \rangle = m(\xddot_x(t) dx + \xddot_y(t) dy)$, where we have used (\ref{equ:elformechanics}). 

We conclude that $m \xddot_x(t) =0$ and $m \xddot_y(t) =0$. In other words, $\frac{d}{dt} m \xdot_x(t)=0$ and $\frac{d}{dt}m \xdot_y(t) = 0$, i.e., the moments in the $x$ and $y$ directions, are conserved.
\end{example}

Before we continue to the general case let us describe how we will specify the transformations that keep the Lagrangian invariant. It turns out that the proper setting are continuous transformations that can be described by their so-called {\em generators}. We follow \cite{Neu11}[Chapter 5].
\begin{definition}[Generators and Invariance] \label{def:generatorsandinvariance}
Let $\Lag(t, \x, \xdot)$ be a Lagrangian and consider the following transformation
\begin{align*}
t \mapsto t' = T(t, \x, \epsilon), \\
\x \mapsto \x' = X(t, \x, \epsilon),
\end{align*}
where for $\epsilon=0$ the transformation is the identity and the transformation is smooth as a function of $\epsilon$. Let
\begin{align} 
T(t, \x, \epsilon) & = t + \tau(t, \x) \epsilon + O\left(\epsilon^2 \right), \nonumber \\
X(t, \x, \epsilon) & = \x + \xi(t, \x) \epsilon + O\left(\epsilon^2\right). \label{equ:generatorsspace}
\end{align}
The terms $\tau(t, \x)$ and $\xi(t, \x)$ are called the {\em generators} of the transformation. We say that the generators leave the Lagrangian invariant if 
\begin{align*}
\Lag(t', \x', \xdot') \frac{dt'}{dt}- \Lag(t, \x, \xdot) = O \left(\epsilon^2 \right).
\end{align*}
\end{definition}
The EL equations (\ref{equ:elequations}) give us condition to be at a stationary point, i.e., a condition that no variation will allow us to decrease or increase the functionally locally up to linear terms. In a similar manner we can derive a condition for a specific variation, namely the one given by the transformation not to change the value of the functional to first order. For a proof of the following lemma we refer the reader to \cite{Neu11}[Chapter 5].
\begin{lemma}[Rund-Trautmann Identity] \label{lem:rundtrautmann}
Let $\Lag(t, \x, \xdot)$ be a Lagrangian and 
consider a transformation with generators $\tau(t, \x)$ and $\xi(t, \x)$. 
If
\begin{align} \label{equ:rundtrautmann}
\left\langle \frac{\partial \Lag}{\partial \x}, \dx \right\rangle +
\left\langle \frac{\partial \Lag}{\partial \xdot}, \dxdot \right\rangle +
\frac{\partial \Lag}{\partial t} \tau +
\left(\Lag - \left\langle \xdot, \frac{\partial \Lag}{\partial \xdot} \right\rangle \right) \dot{\tau}  = 0
\end{align}
then the generators leave the Lagrangian invariant.
\end{lemma}

\begin{lemma}[Noether's Conservation Law] \label{lem:noether}
Let $\Lag(t, \x, \xdot)$ be a Lagrangian that is invariant under the generators $\tau(t, \x)$ and $\dx(t, \x)$. Let $\x(t)$ be an extremal function of $\Act[\x] = \int \Lag(t, \x, \xdot) dt$. Then 
\begin{align} \label{equ:noether}
\left\langle  \frac{\partial \Lag}{\partial \xdot}, \dx \right\rangle + \left( \Lag - \left\langle \xdot, \frac{\partial \Lag}{\partial \xdot} \right\rangle \right) \tau 
\end{align}
is conserved along $\x(t)$.
\end{lemma}
\begin{proof} By assumption $\x(t)$ is an extremal function of $\Act[\x]$ and hence fulfills the EL equations (\ref{equ:elequations}) with $\nabla \Act$ given by (\ref{equ:variationalderivative}). A fortiori, for any generator $\dx$ we therefore have
\begin{align} \label{equ:eliform}
\left\langle \frac{\partial \Lag}{\partial \x}, \dx \right\rangle = 
\left\langle \frac{d}{dt} \frac{\partial \Lag}{\partial \xdot}, \dx \right\rangle.
\end{align}
Further, expand $\frac{d \Lag}{d t}$ and use the extremality to get
\begin{align*}
 \frac{d \Lag}{d t} & = \frac{\partial \Lag}{\partial t} + \left\langle \frac{\partial \Lag}{\partial \x}, \xdot \right\rangle + \left\langle \frac{\partial \Lag}{\partial \xdot}, \xddot \right\rangle 
= \frac{\partial \Lag}{\partial t} + \left\langle \frac{d}{dt} \frac{\partial \Lag}{\partial \xdot}, \xdot \right\rangle + \left\langle \frac{\partial \Lag}{\partial \xdot}, \xddot \right\rangle  
= \frac{\partial \Lag}{\partial t} + \frac{d}{d t} \left\langle \frac{\partial \Lag}{\partial \xdot}, \xdot \right\rangle.
\end{align*}
This can be written as 
\begin{align} \label{equ:eliiform}
\frac{\partial \Lag}{\partial t} = \frac{d}{dt} \left[\Lag - \left\langle \xdot, \frac{\partial \Lag}{\partial \xdot} \right\rangle \right].
\end{align}
By assumption the generators $\tau(t, \x)$ and $\dx(t, \x)$ keep $\Act[\x(t)]$ invariant. Note that the left-hand expressions in (\ref{equ:eliform}) and (\ref{equ:eliiform}) appear in the Rund-Trautmann conditions stated in Lemma~\ref{lem:rundtrautmann}. If in the Rund-Trautmann condition we replace those left-hand expressions by their equivalent right-hand expressions we get
\begin{align*}
0 & = \left\langle \frac{\partial \Lag}{\partial \x}, \dx \right\rangle +
\left\langle \frac{\partial \Lag}{\partial \xdot}, \dxdot \right\rangle +
\frac{\partial \Lag}{\partial t} \tau +
\left( \Lag - \left\langle \xdot, \frac{\partial \Lag}{\partial \xdot} \right\rangle \right) \dot{\tau} \\
& = \left\langle \frac{d}{dt} \frac{\partial \Lag}{\partial \xdot}, \dx \right\rangle +
\left\langle \frac{\partial \Lag}{\partial \xdot}, \dxdot \right\rangle +
\frac{d}{dt} \left[ \Lag - \left\langle \xdot, \frac{\partial \Lag}{\partial \xdot} \right\rangle \right] \tau +
\left(\Lag - \left\langle \xdot, \frac{\partial \Lag}{\partial \xdot} \right\rangle\right) \dot{\tau} \\
& =  \frac{d}{dt} \left[\left\langle  \frac{\partial \Lag}{\partial \xdot}, \dx \right\rangle\right] +
\frac{d}{dt} \left[\left( \Lag - \left\langle \xdot, \frac{\partial \Lag}{\partial \xdot} \right\rangle \right) \tau \right] \\
& =  \frac{d}{dt} \left[\left\langle  \frac{\partial \Lag}{\partial \xdot}, \dx \right\rangle+\left( \Lag - \left\langle \xdot, \frac{\partial \Lag}{\partial \xdot} \right\rangle \right) \tau \right]
\end{align*}
as promised. 
\end{proof}

Note the following two special cases. If $\tau(t, \x) = 0$ then the invariant quantity is $\left\langle  \frac{\partial \Lag}{\partial \xdot}, \dx \right\rangle$. If further $\dx(t, \x) = \dx$, a constant, then each of the components of $\frac{\partial \Lag}{\partial \xdot}$ for which $\dx$ is not zero is invariant by itself.  If $\xi(t, \x) = 0$ then the invariant quantity is $\left( \Lag - \left\langle \xdot, \frac{\partial \Lag}{\partial \xdot} \right\rangle \right) \tau $.  If further $\tau(t, \x) = \tau$, a constant, then the invariant quantity is $\Lag - \left\langle \xdot, \frac{\partial \Lag}{\partial \xdot} \right\rangle$. This last expression is called the {\em Hamiltonian}.

\begin{example}[Mechanics Contd]\label{exp:hamiltonian}
As we previously discussed, in our running example the Lagrangian does not explicitly depend on time $t$. Hence, by our previous discussion the Hamiltonian is a conserved quantity. We have
\begin{align*}
 H = \Lag - \left\langle \xdot, \frac{\partial \Lag}{\partial \xdot} \right\rangle  = \frac12 m \|\xdot(t)\|^2 -\pot(\x(t)) - m \langle \xdot, \xdot \rangle = - \left(\frac12 m \|\xdot(t)\|^2 +\pot(\x(t)) \right),
\end{align*}
confirming our previous result that the sum of the kinetic and potential energy stays preserved.
\end{example}

\subsection{Extensions}\label{sec:extensions}
We have seen in the sections above how we can find conserved quantities for systems that are described by a Lagrangian that has an invariance.

For our applications we need to relax some of the assumptions. First, some important examples cannot be described in terms of a Lagrangian directly. In our case this applies e.g., to the case of Gradient Flow, see Section~\ref{sec:learning}. As we will discuss in more detail in Section~\ref{sec:learning}, one way to circumvent this problem is to represent Gradient Flow as a limiting case of a dynamics that does have an associated Lagrangian. We describe a second, more direct approach, here. Second, even if dynamics can be described by a Lagrangian, the symmetries that are important for us might only keep part of the Lagrangian invariant. We will now discuss how in such situations we can nevertheless get useful information by applying a procedure that is very close in spirit to the one employed by Noether.

Let us quickly review how Noether's theorem was derived. We started with an action $\Act$ defined by a Lagrangian $\Lag$. The requirement that a path $\x$ was extremal led to the EL equations (\ref{equ:elequations}). Further, if the Lagrangian exhibited an invariance wrt to some generators $(\tau, \dx)$ then this lead to the Rund-Trautmann equations (\ref{equ:rundtrautmann}). The last step consisted in combining these two sets of equations and to realize that the result can be written as a total derivative. This gave rise to the set of conserved quantities stated in Lemma~\ref{lem:noether}.

For us the starting point will be a differential equation describing the continuous limit of a discrete gradient-like optimization algorithm. Let us write it in the generic form
\begin{align} \label{equ:extensioneulerlagrange}
E(t, \x, \xdot, \xddot) = - \nabla_{\x} \pot(\x).
\end{align}
As we will discuss in much more detail in Section~\ref{sec:learning}, for Gradient Flow the term $E$ has the simple form $E(t, \x, \xdot, \xddot) = \xdot$, whereas for other dynamics it might be a function of $\xddot$ or involve a combination of terms, including factors of $t$. The vector $\x$ represents the set of parameters of our problem. The term $\pot(\x)$ represents the loss of the network. It is a function of the network parameters. For some cases we consider we will have a Lagrangian but e.g. for Gradient Flow there is no Lagrangian for which (\ref{equ:extensioneulerlagrange}) is the Euler-Lagrange equation (although, as we will discuss later, we can think of Gradient Flow as a limiting case for which a Lagrangian exists). Nevertheless, we will think of (\ref{equ:extensioneulerlagrange}) as our Euler-Lagrange equation. This takes care of the first ingredient in our program.

For our applications the invariance will typically only apply to the potential $\pot(\x)$. I.e., there will be a generator $\dx(\x)$ (typically only dependent on the $\x$ but not on $t$ directly) so that $\pot(\x)$ is invariant. Applying the Rund-Trautmann conditions (\ref{equ:rundtrautmann}) for this part of the \say{Lagrangian} we get the condition
\begin{align} \label{equ:extensionrundtrautmann}
\langle \nabla_{\x} \pot, \dx \rangle = 0.
\end{align}
Multiplying both sides of (\ref{equ:extensioneulerlagrange}) by $\dx$ and combining with (\ref{equ:extensionrundtrautmann}) we get the set of equations
\begin{align} \label{equ:extensionconserved}
   \langle E(t, \x, \xdot, \xddot), \dx(\x) \rangle = 0.
\end{align}
In general this expression can not be written as a total derivative and hence we do not get what is typically called a conserved quantity as before. But we can  think of (\ref{equ:extensionconserved}) as a \say{conserved} quantity and often this equation gives rise to interesting bounds on the parameters.

\subsection{From Invariances to Generators - Lie Groups and Lie Algebras}\label{sec:lie}

As described in Section~\ref{sec:extensions}, the differential equations we are interested in are mostly of the form \eqref{equ:extensioneulerlagrange}. Hence the recipe to find 
\say{conserved} quantities is as follows: (i) find a transformation that leaves $\pot(\x)$ invariant; (ii) find the generator $\dx$ associated with this transformation; (iii) evaluate \eqref{equ:extensionconserved} and draw your conclusions. 

In Appendix~\ref{apd:lie} we explain step (ii) in more detail, i.e., how, given a transformation, we find its corresponding generators.
This can always be done in a pedestrian way, expanding (\ref{equ:generatorsspace}) by \say{hand}. But there exists a well-studied area in mathematics (Lie Groups and Lie Algebras) that deals exactly with this issue and so it is worth pointing out the connections.

We will not make  use of the notation and language introduced in Appendix~\ref{apd:lie} in the main part of the paper. The reader can therefore safely skip Appendix~\ref{apd:lie} on a first reading.

%We need two ingredients: a differential equation that governs the evolution - $F(t,\x,\xdot, \dots, \pmb{x}^{(n)}) = 0 \in \R^d$ and a Lie group $G$ that acts on $\R^d$ - the spatial domain of $\x$. Let $\mathfrak{g}$ be the Lie algebra associated with $G$. Then for every $\dx \in \mathfrak{g}$
%$$ \langle F(t,\x,\xdot, \dots, \pmb{x}^{(n)}), \dx (\x) \rangle = 0 .$$

\section{Learning Setup And Optimization Algorithms} \label{sec:learning}
We are given a training set $\left\{\left(\x^{(i)}, \y^{(i)}\right)\right\}_{i = 1}^n \subset \R^{d_x} \times \R^{d_y}$, and want to learn a hypothesis from a parametric family $\mathcal{H} := \left\{\pred_{\w} : \R^{d_x} \xrightarrow{} \R^{d_y} \ \middle| \ \w \in \mathcal{W} \right\}$ by minimizing the empirical loss function
\begin{equation}\label{eq:objective}
\w^* = \text{argmin}_{\w \in \mathcal{W}} \ \loss(\w) \text{,}
\end{equation}
where $\loss$ implicitly depends on the training set. Although the basic idea applies to any parametric family, we will limit ourselves to the case where the elements of $\mathcal{H}$ are represented by NNs with layers numbered from $0$ (input) to $K$ (output), containing $d_x = d_0, d_1, \dots$, and $d_K = d_y$ neurons respectively. The activation functions for the layers $1$ to $K$ are presumed to be $\sigma_1, \dots, \sigma_K : \R \xrightarrow{} \R$. The weight matrices will be denoted by $W^{(1)}, W^{(2)}, \dots, W^{(K)}$, respectively, where matrix  $W^{(k)}$ connects layer $k-1$ to layer $k$. We define 
\begin{equation}\label{eq:defoff}
\pred_{\w}(\x) := \sigma_K \left(W^{(K)} \sigma_{K-1} \left( W^{(K-1)} \cdots \sigma_{1} \left(W^{(1)} \x \right) \right) \right) .
\end{equation}
The dimension of $\w$ is $m = \sum_{i=0}^{K-1} d_i d_{i+1}$. Note that \eqref{eq:defoff} is a network without bias terms. In Section~\ref{sec:biases} we consider networks with bias terms. In this case the network becomes $\pred_{\w}(\x) := \sigma_K (b^{(K)} + W^{(K)} \sigma_{K-1}(b^{(K-1)} + W^{(K-1)} \dots \sigma_{1}(b^{(1)} + W^{(1)} \x )))$. For $n_1,n_2 \in \N$ we denote by $\text{Mat}(n_1,n_2)$ the set of $n_1 \times n_2$ matrices with real entries. We will also use the following notation $\text{Mat}(n) := \text{Mat}(n,n)$. For $n \in \N$ we denote by $\text{SO}(n)$ the special orthogonal group in dimension $n$ and by $\mathfrak{so}(n)$ the algebra of $n \times n$ skew-symmetric matrices.

\paragraph{Continuous vs Discrete.} We are mainly interested in the following three optimization algorithms for minimizing \eqref{eq:objective}:
Newtonian Mechanics (ND), Nesterov's Accelerated Gradient Descent (NAGD) and Gradient Descent (GD). More precisely, we will analyse the continuous-time versions of these algorithms. The connection between a discrete-time optimization algorithm and its corresponding continuous-time version given by its associated \textit{ordinary differential equation} (ODE) has been the topic of a considerable literature. Note that GD becomes Gradient Flow (GF) when we consider the limit of the learning rate going to zero. The analysis of how NAGD becomes Nesterov's Accelerated Gradient Flow (NAGF) is more involved and was discussed in \citet{candesNesterov}. The ND algorithm is not a popular optimization algorithm but it will be helpful to analyze it as it demonstrates our results well and can be understood as a special case of the NAGF. 
%One final word: since ND as well as NAGF involve $\wddot$ we do need to specify $\w(t=0)$ as well as $\wdot(t=0)$. 

Let us briefly explore the connections between continuous time dynamics and their respective discrete time implementations for NAGF and ND. We follow \citet{candesNesterov}. NAGF can take the following form: starting with $\w_0$ and $\y_0 = \w_0$ define
\begin{align}
\w_k &= \y_{k-1} - \eta \nabla \loss(\y_{k-1}), \nonumber \\
\y_k &= \w_k + \frac{k-1}{k+2} (\w_k - \w_{k-1}). \label{eq:nesterovdiscrete}
\end{align}
In the limit of vanishing step sizes this is equivalent to the following ODE
$$
\wddot + \frac{3}{t} \wdot + \nabla \loss(\w) = 0,
$$
where $t \approx k \sqrt{\eta}$ and $\w(t \sqrt{\eta}) \approx \w_k$. 

ND corresponds to 
$$
\wddot + \nabla \loss(\w) =0.
$$
%Let us give a quick sketch of a derivation of a discrete version of ND. Approximating, for a small $\delta$,
%$$
%\wddot(t) \approx \frac{\w(t+2\delta) - 2\w(t+\delta) + \w(t)}{\delta^2}
%$$
%and associating $\w(t\delta) = \w_k$ while evaluating the vector field at the current estimate (which can be thought of as a version of the forward-Euler method \comm{GG}{Can it?}) one gets
A discrete version can be written as
$$
\w_{k+2} = 2\w_{k+1} - \w_{k} - \eta \nabla \loss(\w_{k+1}),
$$
where, as for NAGF,  $t \approx k \sqrt{\eta}$. The above equation can be rewritten in a more familiar form
\begin{equation}
\w_{k+2} = \w_{k+1} + (\w_{k+1} - \w_{k}) - \eta \nabla \loss(\w_{k+1}),
\end{equation}
where we can interpret $(\w_{k+1} - \w_{k})$ as a momentum term without any dampening.

\paragraph{Lagrangians.} 
Rather than analyzing these three continuous-time versions separately, it is more convenient to consider a {\em class} of continuous-time dynamics.
The general form of ODEs we consider is
\begin{equation} \label{equ:ode}
\cnstdd(t) \wddot + \cnstd(t) \wdot + \nabla \loss(\w) = 0.   
\end{equation}
The corresponding Lagrangian is
\begin{equation} \label{equ:lagrangian}
e^{\int_{0}^{t} \frac{\cnstd(\tau)}{\cnstdd(\tau)} d\tau} \left(\frac{1}{2} \| \wdot \|^2 - \frac{1}{\cnstdd(t)} \loss(\w) \right).  
\end{equation}
The relationship between (\ref{equ:lagrangian}) and (\ref{equ:ode}) is quickly established by checking that the variational derivative
(\ref{equ:variationalderivative}) corresponding to (\ref{equ:lagrangian}) is indeed equal to the left-hand side of (\ref{equ:ode}).

A few remarks are in order. First, we recover ND by setting $(\cnstdd(t)= 1, \cnstd(t)=0)$ and NAGF by using $(\cnstdd(t)= 1, \cnstd(t)=3/t)$. No fixed set of parameters corresponds to GF but one can interpret GF as the dynamics for $(\cnstdd(t)= \kappa, \cnstd(t)=1)$ when $\kappa$ tends to $0$. The \say{physics} interpretation of GF is that it is the \say{massless} (strong friction) limit where $\kappa$ is interpreted as mass for the following damped Lagrangian $e^{t/\kappa}\left( \frac{\kappa}{2}\|\wdot\|^2 - \loss(\w)\right)$ (as explained in \citet{villani2008optimal}[p. 646]). If the reader is looking for a gentle introduction and connections to other aspects of optimization we recommend searching for a blog by Andre Wibisono. 
%or for an in depth treatment \citet{villani}. 
This general view-point will allow us to treat all three cases in a uniform manner. Second, in principle we could have included terms of higher order in our dynamics. Our basic framework would easily extend to such a case. But since the resulting dynamics have not attracted much attention to this point we opted to stick to the less general setting. Table~\ref{tab:dynamicsandlagrangians} summarizes the situation.

\begin{table}
\centering
%\begin{tabular}{lcLc} 
\begin{tabularx}\textwidth{@{}lMLc@{}}
    \toprule
    Algorithm & ODE & \multicolumn{1}{l}{} & Lagrangian \\ 
    \midrule
\rule{0pt}{1ex}    
General Dynamics  & \cnstdd(t) \wddot + \cnstd(t) \wdot + \nabla \loss(\w) = 0 & eq:odeGeneral & $e^{\int_{0}^{t} \frac{\cnstd(\tau)}{\cnstdd(\tau)} d\tau} \left(\frac{1}{2} \| \wdot \|^2 - \frac{1}{\cnstdd(t)} \loss(\w) \right)$ \\ 
\rule{0pt}{3ex}
Newtonian Dynamics  & \wddot + \nabla \loss(\w) = 0 & eq:odeND & $\frac{1}{2} \| \wdot \|^2 - \loss(\w)$ \\ 
\rule{0pt}{3ex}
Nesterov's Accelerated GF & \wddot + \frac{3}{t} \wdot + \nabla \loss(\w) = 0 & eq:odeNAGF & $t^3 \left(\frac{1}{2} \| \wdot \|^2 - \loss(\w) \right)$  \\
\rule{0pt}{3ex}
Gradient Flow & \wdot + \nabla \loss(\w) = 0 & eq:odeGF &  \\
    \bottomrule
\end{tabularx}
\caption{\label{tab:dynamicsandlagrangians} The general form of the ODE and it's Lagrangian as well as the three special cases of interest.}
\end{table}

\section{Homogeneity of Activation Leads To Bounded Weights}
It is now time to look at concrete instances of our general framework -- how symmetries/invariances give rise to conserved quantities.  We look at conserved quantities due to properties of the activation functions, the special case of linear networks, invariances due to data symmetry, and time invariance. We treat each of these cases in a separate section. In each section we follow the same structure. We (i) identify the symmetry/invariance, (ii) find the corresponding generators, (iii) deduce from the generators the conserved quantities, and (iv) discuss the implications.

We start by looking at symmetries due to special properties of the activation function. In particular, we consider the following activation functions:
\begin{itemize}
    \item ReLU: $x \mapsto \max(0,x)$,
    \item For $\alpha < 1$, LeakyReLU($\alpha$): $x \mapsto         \begin{cases}
        x &\text{if $ x > 0$}\\
        \alpha x &\text{if $ x \leq 0$}
        \end{cases} $, 
    \item For $p \in \R_+$, Polynomial: $x \mapsto x^p$,
    \item For $p \in \R_+$, Rectified Polynomial Unit (RePU($p$)): $x \mapsto         \begin{cases}
        x^p &\text{if $ x > 0$}\\
        0 &\text{if $ x \leq 0$}
        \end{cases} $,
    \item For $\beta \in \R_+$, Swish($\beta$): $x \mapsto \frac{x}{1 + e^{-\beta x}}$.
\end{itemize}

ReLU is perhaps the most popular activation function used in practice. The LeakyReLU is closely connected to ReLU. Polynomial activations, in particular quadratic activations, were analyzed from a theoretical point of view in \citet{lenkaQuadratic, allenzhu20backward}.  The RePU is a natural combination of ReLU and Polynomial activations. Note that ReLU = RePU($1$). We treat them as separate cases due to importance of ReLU. Swish was introduced in \citet{swish}, where it was argued that it performs better than ReLU on some learning tasks.

The symmetries we consider in this section rely heavily on the \textit{homogeneity} of these activation functions.
\begin{definition}[Homogeneity]
For $p \in \R_+$, we say that a function $\sigma : \R \rightarrow \R$ is $p$-homogeneous if for every $\lambda \in \R_+$ and $x \in \R$, $$
\sigma(\lambda \cdot x) = \lambda^p \cdot \sigma(x).
$$
\end{definition}
\begin{observation}\label{obs:homogeneity}
For every $p \in \R_+$, Polynomial($p$) and RePU($p$) are $p$-homogeneous and $\text{ReLU}= \text{RePU}($p=1$)$ is  $1$-homogeneous. Swish is not homogeneous but it satisfies the following related identity: for every $x \in \R$ and $\lambda, \beta \in \R_+$,
$$
\text{Swish}(\beta)(\lambda \cdot x) = \lambda \cdot \text{Swish}(\lambda \cdot \beta)(x).
$$
\end{observation}
We start with deriving conserved quantities for $1$-homogeneous functions, which means they will be conserved for ReLU activations. We then consider $p$-homogeneous and Swish activation functions.

\begin{table}
\centering
\begin{tabular}{SS} \toprule
    { }                           & {Gradient Flow}    \\ \midrule \rule{0pt}{1ex}
    {(Leaky)ReLU}  & {$\|W^{(h)}\|_F^2 - \|W^{(h+1)}\|_F^2 = \text{const}$}  \\ \rule{0pt}{3ex} 
    {$x^p$ / RePU($p$)}           & {$\|W^{(h)}\|_F^2 - p \cdot \|W^{(h+1)}\|_F^2 = \text{const}$}   \\ \rule{0pt}{3ex} 
    {Swish}  & {$\|W^{(h)}\|_F^2 - \|W^{(h+1)}\|_F^2 - \|\beta^{(h)}\|^2_2 = \text{const}$}  \\ 
    %5  & 1.29   & +0.099 & 1.29   \\
    %6  & 0.483  & -0.183 & 0.483  \\
    %7  & 0.766  & -0.475 & 0.766  \\
    %8  & 0.624  & +0.365 & 0.624  \\ \midrule
    %9  & 0.641  & -0.466 & 0.641  \\
    %10 & 0.45   & +0.421 & 0.45   \\
    %11 & 0.598  & -0.597 & 0.598  \\ 
    \bottomrule
\end{tabular}

\rule{0pt}{2ex}    

\centering
\begin{tabular}{SS} \toprule
    { }                           & {Newtonian Dynamics}    \\ \midrule \rule{0pt}{1ex}
    {(Leaky)ReLU}  & {$\langle W^{(h)},\ddot{W}^{(h)} \rangle - \langle W^{(h+1)},\ddot{W}^{(h+1)} \rangle = 0$}  \\ \rule{0pt}{3ex} 
    {$x^p$ / RePU($p$)}           & {$\langle W^{(h)},\ddot{W}^{(h)} \rangle - p \cdot \langle W^{(h+1)}, \ddot{W}^{(h+1)} \rangle = 0$}  \\ \rule{0pt}{3ex} 
    {Swish}  & {$\langle W^{(h)}, \ddot{W}^{(h)} \rangle - \langle W^{(h)}, \ddot{W}^{(h+1)} \rangle - \langle \beta^{(h)}, \ddot{\beta}^{(h)} \rangle = 0$}   \\ 
    \bottomrule
\end{tabular}

\rule{0pt}{2ex} 

\centering
\begin{tabular}{SS} \toprule
    { }                           & {Nesterov's Accelerated Gradient Flow}    \\ \midrule \rule{0pt}{1ex}
    {(Leaky)ReLU}  & {$\langle W^{(h)}, \frac{3}{t} \dot{W}^{(h)} + \ddot{W}^{(h)} \rangle - \langle W^{(h+1)}, \frac{3}{t} \dot{W}^{(h+1)} + \ddot{W}^{(h+1)} \rangle = 0$}  \\ \rule{0pt}{3ex} 
    {$x^p$ / RePU($p$)}           & {$\langle W^{(h)}, \frac{3}{t} \dot{W}^{(h)} + \ddot{W}^{(h)} \rangle - p \cdot \langle W^{(h+1)}, \frac{3}{t} \dot{W}^{(h+1)} + \ddot{W}^{(h+1)} \rangle = 0$}  \\ \rule{0pt}{3ex} 
    {Swish}  & {$\langle W^{(h)}, \frac{3}{t} \dot{W}^{(h)} + \ddot{W}^{(h)} \rangle - \langle W^{(h+1)}, \frac{3}{t} \dot{W}^{(h+1)} + \ddot{W}^{(h+1)} \rangle - \langle \beta^{(h)}, \frac{3}{t}\dot{\beta}^{(h)} + \ddot{\beta}^{(h)} \rangle = 0$}   \\ 
    \bottomrule
\end{tabular}

\rule{0pt}{2ex} 

\centering
\begin{tabular}{SS} \toprule
    { }                           & {General Dynamics}    \\ \midrule \rule{0pt}{1ex}
    {(Leaky)ReLU}  & {$\langle W^{(h)}, \cnstd \dot{W}^{(h)} + \cnstdd \ddot{W}^{(h)} \rangle - \langle W^{(h+1)}, \cnstd \dot{W}^{(h+1)} + \cnstdd \ddot{W}^{(h+1)} \rangle = 0$}  \\ \rule{0pt}{3ex} 
    {$x^p$ / RePU($p$)}           & {$\langle W^{(h)}, \cnstd \dot{W}^{(h)} + \cnstdd \ddot{W}^{(h)} \rangle - p \cdot \langle W^{(h+1)}, \cnstd \dot{W}^{(h+1)} + \cnstdd \ddot{W}^{(h+1)} \rangle = 0$}  \\ \rule{0pt}{3ex} 
    {Swish}  & {$\langle W^{(h)}, \cnstd \dot{W}^{(h)} + \cnstdd \ddot{W}^{(h)} \rangle - \langle W^{(h+1)}, \cnstd \dot{W}^{(h+1)} + \cnstdd \ddot{W}^{(h+1)} \rangle - \langle \beta^{(h)}, \cnstd \dot{\beta}^{(h)} + \cnstdd \ddot{\beta}^{(h)} \rangle = 0$}   \\ 
    \bottomrule
\end{tabular}

\caption{\label{tab:homogeneity} Conserved quantities due to the homogeneity of the activation functions. Note that 
we write the conserved quantities {\em per layer}. In fact, these  quantities are conserved {\em per node}, see e.g. (\ref{equ:newtonianpernode}). The reason we write the conserved quantities in this form is that these expressions are slightly simpler and that it makes it easier to compare to other cases that we discuss later where we do not have conservation laws on this more detailed scale.}

\end{table}

\subsection{$1$-homogeneous}\label{sec:1homog}

\paragraph{Symmetry.} Let $h \in [K-1]$, $i \in [d_h]$, and $\epsilon > 0$ and assume that $\sigma_h$, the activation in the $h$-th layer, is $1$-homogeneous. No assumptions are made with respect to the activation functions on any other layer. Indeed, those can be chosen independently. Define $\w_\epsilon $ to be equal to $\w$ apart from 
$$W^{(h)}_\epsilon[i,:] := (1+\epsilon) \cdot W^{(h)}[i,:] ,$$
$$W^{(h+1)}_{\epsilon}[:,i] := \frac{1}{1+\epsilon} \cdot W^{(h+1)}[:,i] \text{.}$$
Then for $f$ as defined in \eqref{eq:defoff}
$$\pred_{\w} \equiv \pred_{\w_\epsilon} \text{.}$$ 
\begin{proof}
Let $a^{(h-1)} \in \R^{d_{h-1}}$ be the activations in layer $h-1$. Note that $
\epsilon$ influences the weight matrices only on level $h$ and $h+1$. Therefore, in order to show that $\pred_{\w} \equiv \pred_{\w_\epsilon}$ it suffices to show that 
\begin{align*}
&W^{(h+1)}_\epsilon \sigma_h \left(W^{(h)}_\epsilon a^{(h-1)} \right) - W^{(h+1)} \sigma_h \left(W^{(h)} a^{(h-1)} \right)  \\
&= \sigma_h \left(\left\langle W^{(h)}_\epsilon[i,:], a^{(h-1)} \right\rangle \right) W^{(h+1)}_\epsilon[:,i] - \sigma_h \left(\left\langle W^{(h)}[i,:], a^{(h-1)} \right\rangle \right) W^{(h+1)}[:,i] \\
&= \sigma_h \left(\left\langle (1+
\epsilon) \cdot W^{(h)}[i,:], a^{(h-1)} \right\rangle \right) \frac{1}{1+\epsilon} \cdot W^{(h+1)}[:,i] - \sigma_h \left(\left\langle W^{(h)}[i,:], a^{(h-1)} \right\rangle \right) W^{(h+1)}[:,i] \\
&= 0,
\end{align*}
where we used the $1$-homogeneity of $\sigma_h$ in the last equation.
\end{proof}

\paragraph{Generator.} Observe that the generator of this symmetry can be associated with a matrix $\dx \in \text{Mat}(m)$ with nonzero entries only on the diagonal. More concretely:
\begin{equation}\label{eq:generator1homog}
\dx[j,j] := \begin{cases}
        1 & \text{if the $j$-th index corresponds to a weight in $W^{(h)}[i,:]$}\\
        -1 & \text{if the $j$-th index corresponds to a weight in $W^{(h+1)}[:,i]$} \\
        0 &\text{otherwise}
        \end{cases}. 
\end{equation}
To see that this is in fact the generator one can approximate to the first order (linearize the transformation) in $\epsilon$:
$$W^{(h)}_\epsilon[i,:] =  W^{(h)}[i,:] + \epsilon \cdot W^{(h)}[i,:] ,$$
$$W^{(h+1)}_{\epsilon}[:,i] := \frac{1}{1+\epsilon} \cdot W^{(h+1)}[:,i] \approx W^{(h+1)}[:,i] - \epsilon \cdot W^{(h+1)}[:,i] \text{,}$$
and see that the non-identity part is equal to $\epsilon \cdot \dx \cdot \w$. 

For a derivation of generators in terms of Lie theory we refer the reader to Appendix~\ref{apd:liehomogeneity}.

\paragraph{Conserved Quantity.} Now where we know the generator we can apply it to the dynamics defined in Section~\ref{sec:learning}. Since all those dynamics are of the form of \eqref{equ:extensioneulerlagrange} we can find the conserverd quantities by inserting our generator into \eqref{equ:extensionconserved}. Applying \eqref{equ:extensionconserved} for the dynamics \eqref{equ:ode} where the generator $\dx$ is defined according to \eqref{eq:generator1homog} we get
\begin{align}
0 
&=\langle E(t, \w, \wdot, \wddot), \dx(\w) \rangle \nonumber \\
&= \langle \cnstdd \wddot + \cnstd \wdot, \dx(\w) \rangle \nonumber \\
&= \langle W^{(h)}[i,:],\cnstdd \ddot{W}^{(h)}[i,:] + \cnstd \dot{W}^{(h)}[i,:] \rangle - \langle W^{(h+1)}[:,i],\cnstdd \ddot{W}^{(h+1)}[:,i] + \cnstd \dot{W}^{(h+1)}[:,i] \rangle . \label{equ:newtonianpernode}
\end{align}
Equation (\ref{equ:newtonianpernode}) is the  {\em per neuron} \say{conserved} quantity for $1$-homogeneous activation functions.  To make this quantity easier to compare to the ones in other sections we derive a {\em per layer} conservation law by summing (\ref{equ:newtonianpernode}) over all neurons in the $h$-th layer 
\begin{equation}\label{eq:1homogconserved}
0 = \langle W^{(h)},\cnstdd \ddot{W}^{(h)} + \cnstd \dot{W}^{(h)} \rangle - \langle W^{(h+1)},\cnstdd \ddot{W}^{(h+1)} + \cnstd \dot{W}^{(h+1)} \rangle.
\end{equation}
This result is the template for entries in Table~\ref{tab:homogeneity}. We observe that for the case when $\cnstdd = 0, \cnstd = 1$ (corresponding to GF) the expression \eqref{eq:1homogconserved} can be written as a total derivative. This yields the conserved quantity
\begin{equation}
\|W^{(h)}\|_F^2 - \|W^{(h+1)}\|_F^2 = \text{const}.    
\end{equation}

\begin{figure}[h]
\centering

 \subfloat[GF]{\includegraphics[width=0.33\columnwidth]{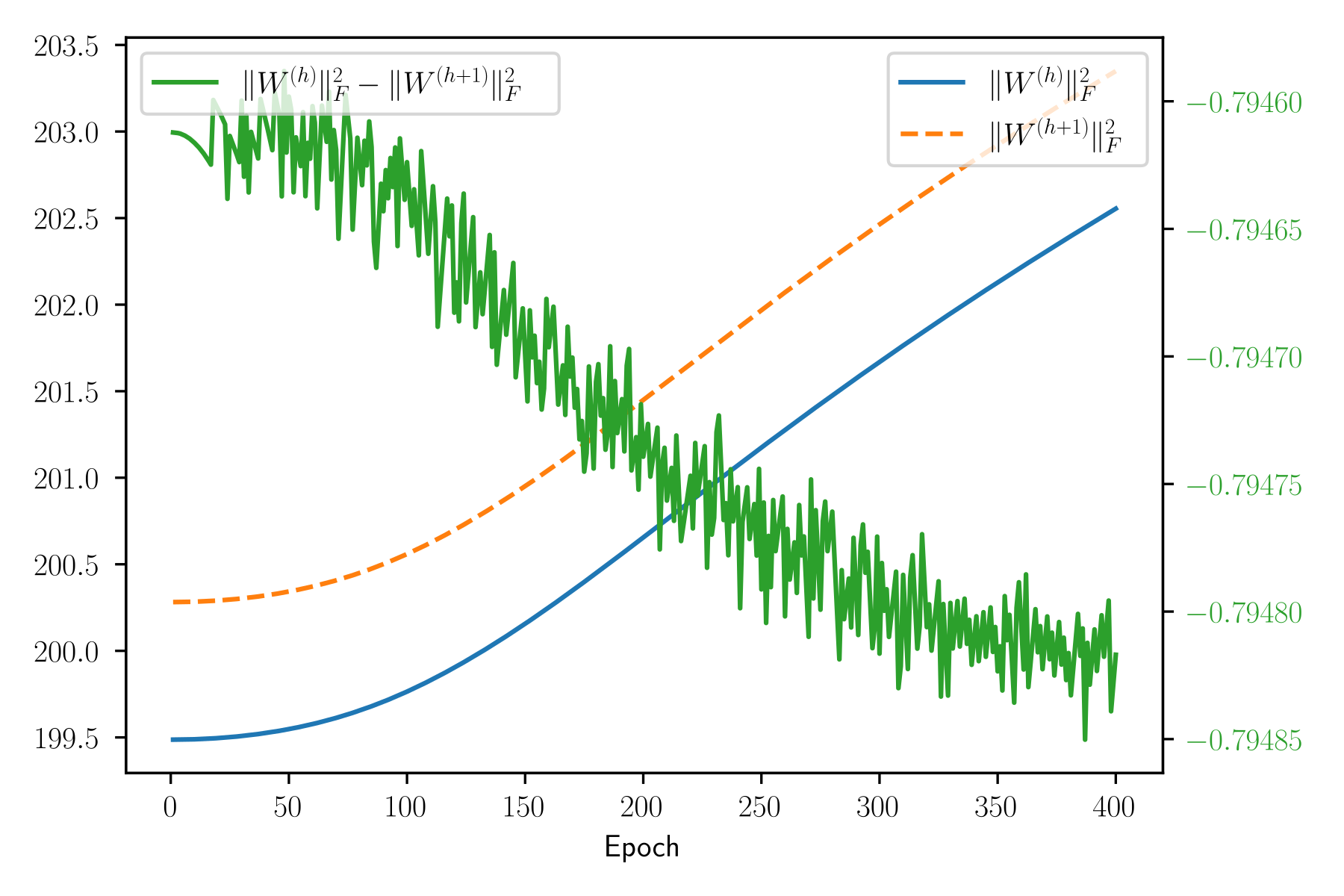}}
 \subfloat[ND]{\includegraphics[width=0.342\columnwidth]{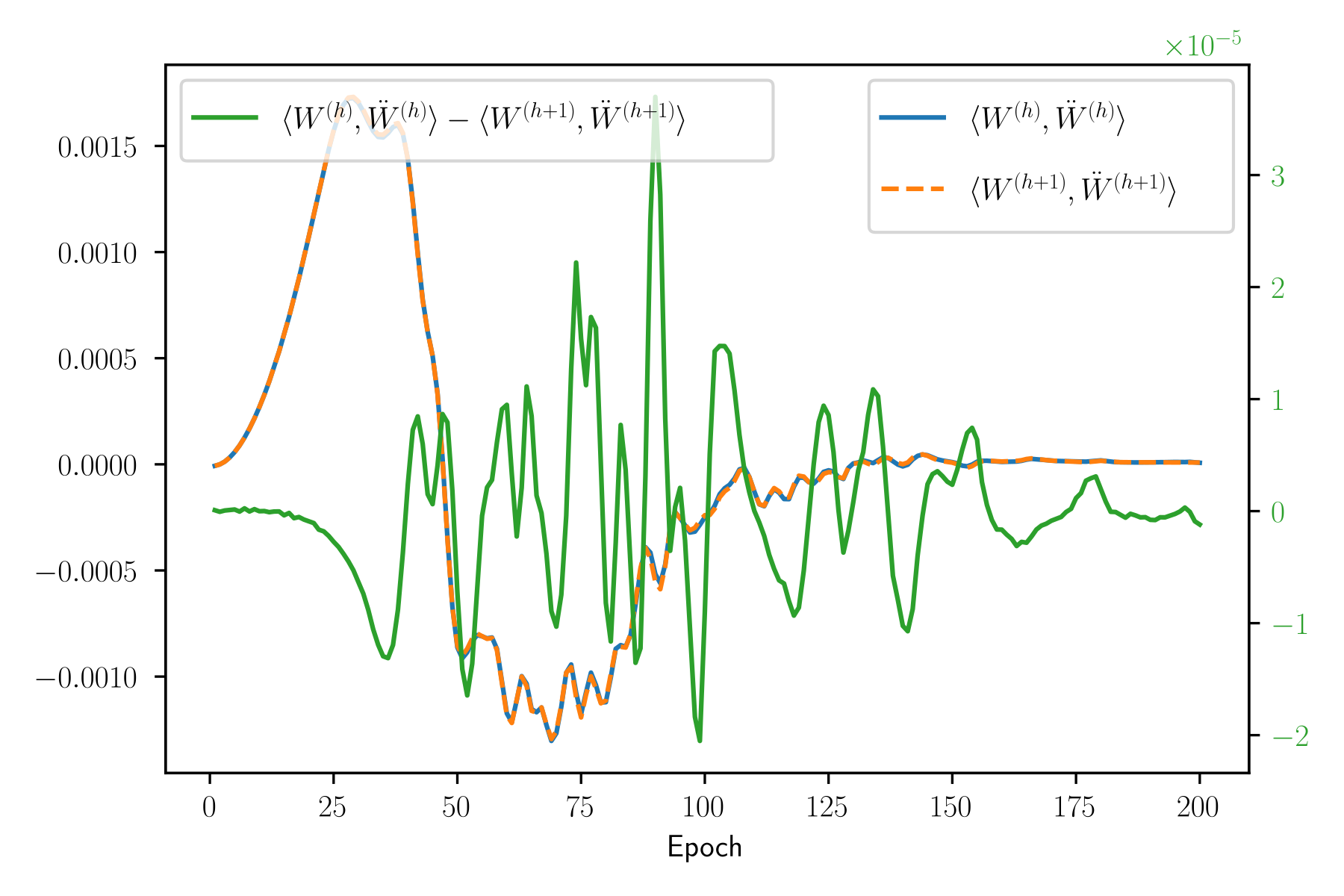}}
 \subfloat[NAGF]{\includegraphics[width=0.33\columnwidth]{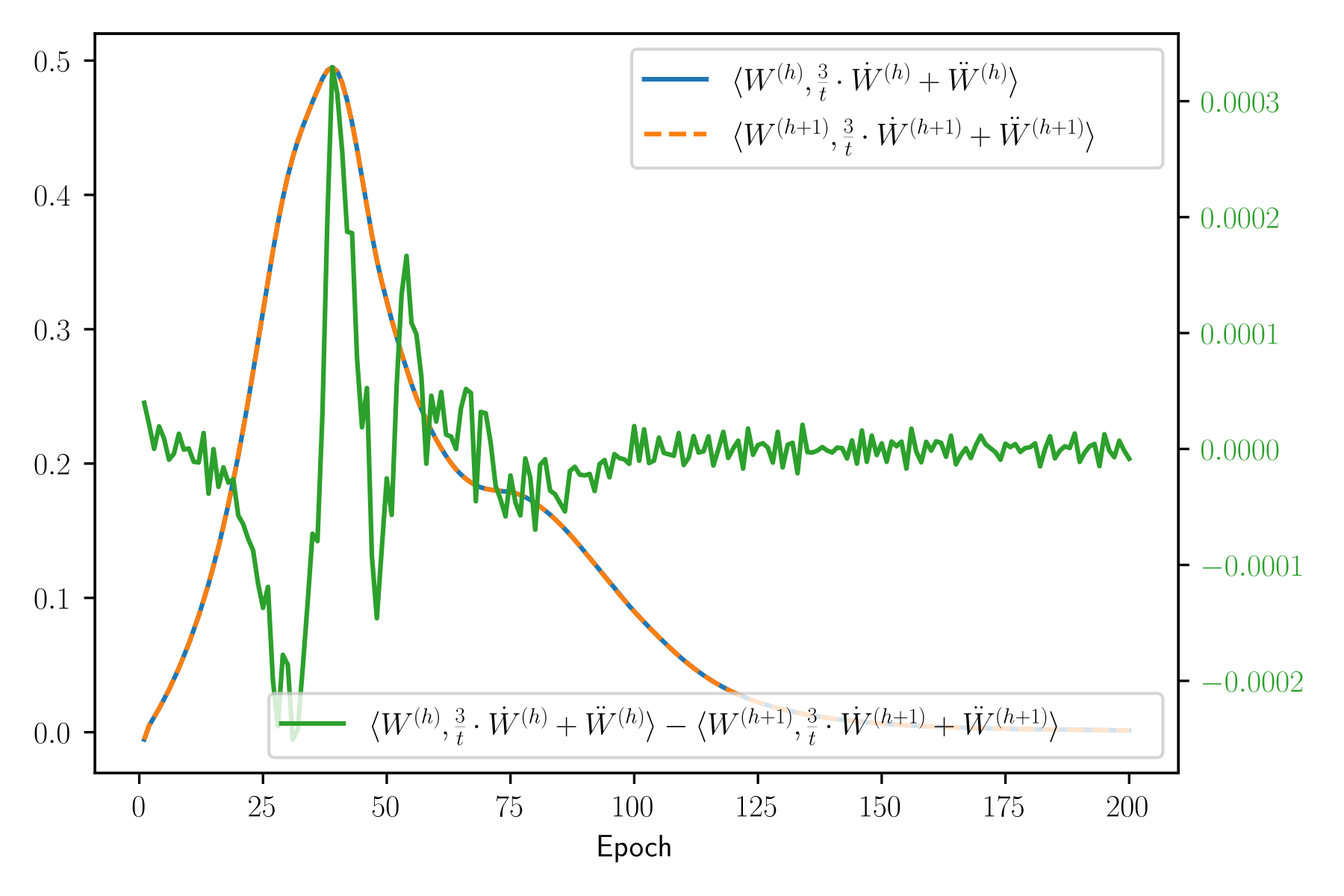}}
 
 \caption{Experiments for $1$-homogeneous activation functions for our main three optimization algorithms.}\label{fig:1hom}
\end{figure}

\paragraph{Implications.} The conserved quantities for $1$-homogeneous activation functions can be understood as bounds on differences between norms of weight matrices. E.g., as we saw, for GF, $\|W^{(h)}\|_F^2 - \|W^{(h+1)}\|_F^2$ stays conserved during the training process. Figure~\ref{fig:1hom} shows the result of an experiment confirming this conservation law (note the different scales of left and right y-axes). The parameters of the experiment are discussed at the end of this section. The equivalent entries in Table~\ref{tab:homogeneity} for ND and NAGF are shown also in Figure~\ref{fig:1hom}. This result for the GF was discovered by \citet{weihubalance}, where they showed that this bound is an important ingredient in guaranteeing the convergence of GD/SGD. Further, in \citet{weihubalance} the question was posed if similar conserved quantities exist for the optimization algorithms introduced in Section~\ref{sec:learning}.

We are now in the position to answer this question in the affirmative. We closely follow the lead of \citet{weihubalance} in our subsequent argument.  Recall that our symmetry stems from the fact that one can multiply one weight matrix by a constant $c$ and divide another weight matrix by the same constant without changing the prediction function. This scaling property poses difficulty when analyzing GD since it implies that the parameters are not constrainted to a compact set. As a consequence, GD and SGD are not guaranteed to converge [\citet{jordanGDdoesntconverge}, Proposition 4.11]. In the same spirit, \citet{shamirhowtoprove} showed that the most important barrier for showing algorithmic results is that the parameters are possibly unbounded during training.

Our aim is to derive bounds on $\|W^{(h)}\|_F^2 - \|W^{(h+1)}\|_F^2$ also for ND and NAGF. In general, the expressions on the right hand side of \eqref{eq:1homogconserved} cannot be written as a total derivatives. But often it is still possible to derive useful bounds.

Let us consider ND. We have
\begin{align*}
&\frac{d}{dt} \left(\langle W^{(h)},\dot{W}^{(h)} \rangle - \langle W^{(h+1)},\dot{W}^{(h+1)} \rangle \right) \\
&= \|\dot{W}^{(h)}\|_F^2 - \|\dot{W}^{(h+1)}\|_F^2 +  
\underbrace{\langle W^{(h)},\ddot{W}^{(h)} \rangle - \langle W^{(h+1)},\ddot{W}^{(h+1)} \rangle}_{=0, \text{see \eqref{eq:1homogconserved}, with $\cnstdd=1$ and $\cnstd=0$}}  \\
& = \|\dot{W}^{(h)}\|_F^2 - \|\dot{W}^{(h+1)}\|_F^2  \\
&\leq  \|\dot{W}^{(h)}\|_F^2 + \|\dot{W}^{(h+1)}\|_F^2.
\end{align*}
Note that  $\|\dot{W}^{(h)}\|_F^2 + \|\dot{W}^{(h+1)}\|_F^2$ can be interpreted as components of the \say{kinetic energy} of the system.

We need a second invariance to proceed. Note that the Lagrangian of ND does not depend on time and thus the Hamiltonian $\frac12 \|\wdot\|^2 + \loss(\w)$, that is the kinetic plus the potential energy of the system, stays preserved as well (see Example~\ref{exp:hamiltonian})! Combining our observations we have  
\begin{align}
& \frac{d^2}{dt^2} \left(\sum_{h=1}^{K-1} \left|\frac12\|W^{(h)}\|_F^2 - \frac12 \|W^{(h+1)}\|_F^2 \right|\right) \nonumber \\
&\leq \sum_{h=1}^{K-1} \left| \frac{d}{dt} \left(\langle W^{(h)},\dot{W}^{(h)} \rangle - \langle W^{(h+1)},\dot{W}^{(h+1)} \rangle \right) \right| \nonumber \\
&\leq \sum_{h=1}^{K-1} \left| \|\dot{W}^{(h)}\|_F^2 + \|\dot{W}^{(h+1)}\|_F^2 \right| \nonumber \\
&\leq 2\|\wdot\|^2 \nonumber \\
&= \left(4\loss(\w) + 2\|\wdot\|^2 \right) - 4\loss(\w)   \nonumber\\
&= 4\loss(\w(0)) + 2\|\wdot(0)\|^2 - 4\loss(\w) &&\text{Hamiltonian stays conserved} \nonumber \\
&\leq 4\left(\loss(\w(0)) - \loss(\w) \right). &&\text{assuming $\wdot(0) = 0$} \nonumber \\
&\leq 4\left(\loss(\w(0)) - \loss^* \right) &&\text{where } \loss^* = \inf_{\x} \loss(\x) \label{eq:NDsecondderivofweights}
\end{align}

Equation \eqref{eq:NDsecondderivofweights} bounds the evolution of the difference between Frobenius norms of weight matrices. We can write \eqref{eq:NDsecondderivofweights} as
\begin{equation}\label{eq:NDsecondderivativebyinitit}
\frac{d^2}{dt^2} \left(\sum_{h=1}^{K-1} \left|\frac12\|W^{(h)}\|_F^2 - \frac12 \|W^{(h+1)}\|_F^2 \right| \right) \leq 4(\loss(\w(0)) - \loss^*) \text{.}
\end{equation}
We see that the difference of norms no longer stays constant for this optimization algorithm. But we do have a bound on the evolution of this difference in terms of the loss at initialization - a fixed parameter that is easily computable. By integrating \eqref{eq:NDsecondderivativebyinitit} we get
\begin{equation}\label{eq:ndhomogfinal}
\left[\sum_{h=1}^{K-1} \left| \|W^{(h)}(t)\|_F^2 - \|W^{(h+1)}(t)\|_F^2 \right| \right] - \left[ \sum_{h=1}^{K-1} \left| \|W^{(h)}(0)\|_F^2 - \|W^{(h+1)}(0)\|_F^2 \right| \right] \leq O(\left(\loss(\w(0)) - \loss^* \right) \cdot t^2).
\end{equation}

This means that if the difference between the norms is small at initialization (which happens for popular initialization schemes like $\frac{1}{\sqrt{d}}$ Gaussian initialization or Xavier/Kaiming when the dimensions of the matrices are the same) then it will remain bounded by the expression on the right hand side. 
%{\color{red} YOu can even make it go slower if you do some kind of restarting. If at some point in learning you just set velocities to $0$ then the second derivative becomes zero. Not sure if it's intersting. You can bound:
%$$
%\sum_{i=1}^K \left| \|W^{(i)}\|_F^2 - \|W^{(i+1)}\|_F^2 %\right|
%$$
%by the same quanitty. Not sure if we want to got there.
%}

For NAGF if we analyze how $\frac12 \|\wdot\|^2 + \loss(\w)$, the \say{energy} of the system , evolves we see that it decreases. Namely,
\begin{align}
&\frac{d}{dt} \left[ \frac12 \|\wdot\|^2 + \loss(\w) \right] \nonumber \\
&= \langle \wdot, \wddot \rangle + \langle \nabla \loss(\w), \wdot \rangle \nonumber  \\
&=- \frac{3}{t} \|\wdot\|^2. && \text{applying }\eqref{eq:odeNAGF} \label{eq:nesterovenergydecreases}
\end{align}
This bound is of interest on it's own. We leave it as an open question if it can be used to derive a bound on $\sum_{h=1}^{K-1} \left|\frac12\|W^{(h)}\|_F^2 - \frac12 \|W^{(h+1)}\|_F^2 \right|$, similarly to what was done for the case of ND.

\paragraph{Experiments.} We employ the KMNIST dataset for experiments. In fact, for complexity reasons we use a subset of the training dataset of size $1000$ for training. We run GD, ND, and NAGD (as described in Section~\ref{sec:learning}) with a learning rate of $0.01$ using negative log-likelihood loss. The architecture used is a fully connected NN with ReLU activation functions with the following dimensions $d_x = d_0 = 784, d_1 = 200, d_2 = 200, d_3 = 200, d_y = d_4 = 10$. We plot the quantities for $h=2$, corresponding to $W^{(2)}$ and $W^{(3)}$.

\subsection{$p$-homogeneous}
Note that $p$-homogeneous activation functions are a generalization of the $1$-homogeneous case. 
The symmetry in this case reads
$$W^{(h)}_\epsilon[i,:] := (1+\epsilon) \cdot W^{(h)}[i,:] ,$$
$$W^{(h+1)}_{\epsilon}[:,i] := \frac{1}{(1+\epsilon)^p} \cdot W^{(h+1)}[:,i] \approx (1 - p \cdot \epsilon) \cdot W^{(h+1)}[:,i] \text{.}$$
The derivation of the conserved quantities are straightforward generalizations of the derivations for the $1$-homogeneous case. The results are summarized in Table~\ref{tab:homogeneity}.

\begin{figure}[ht!]
    \centering 
    \includegraphics[width=0.4\textwidth]{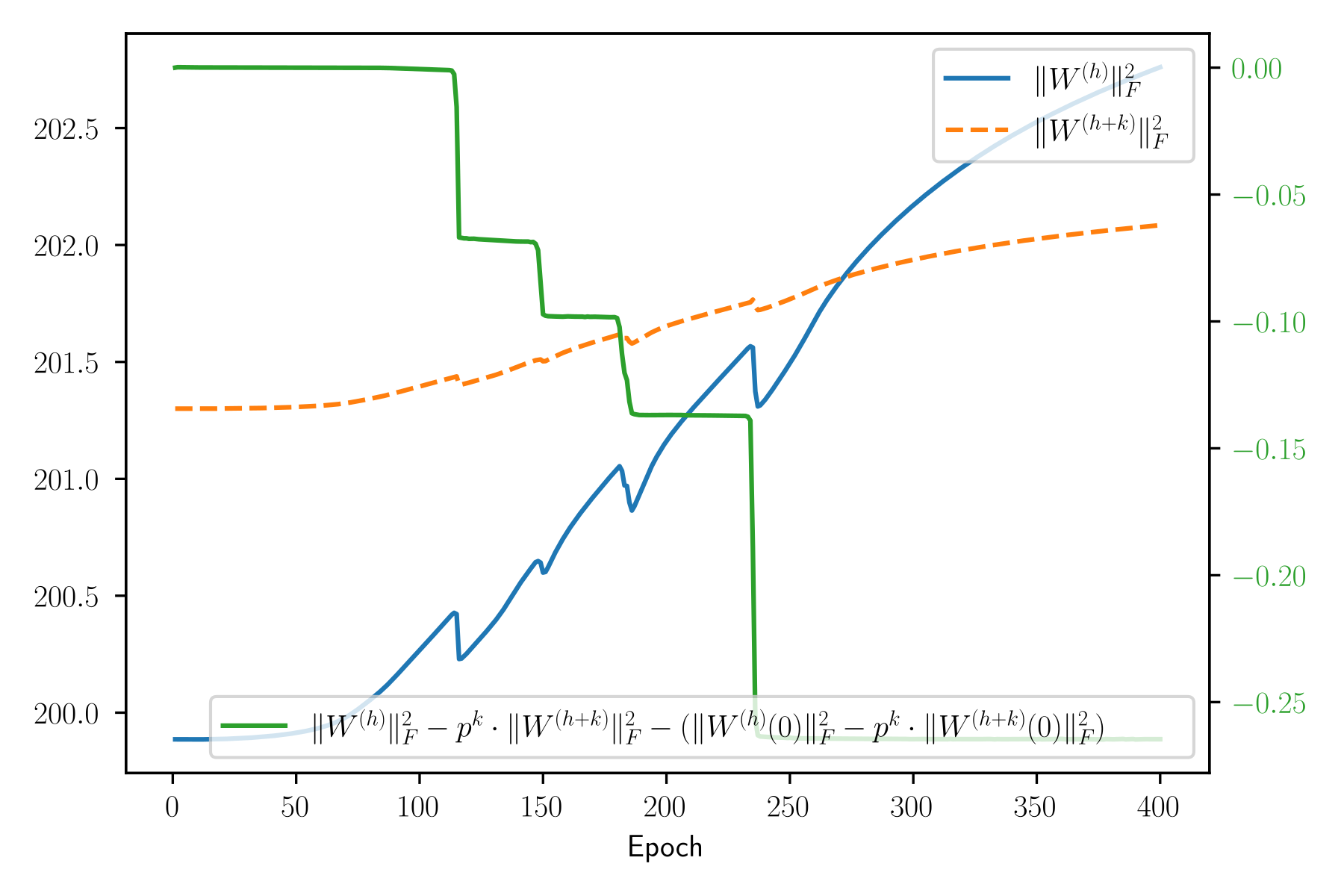}
    \caption{\label{fig:phomGF} Experiment for $p$-homogeneous activation function and GF ($p = 2, k = 2$).}
    \label{fig:phomog}
\end{figure}

\paragraph{Implications.} 
%We will show that the conserved quantities we derived give new guidelines for how to initialize weights when a $p$-homogeneous activation function is used.

If the activation functions between layer $h$ and $h+k$ are all $p$-homogeneous then for the GF we have
$$
\|W^{(h)}(0)\|_F^2 - \|W^{(h)}(t)\|_F^2 = p^k \left[ \|W^{(h+k)}(0)\|_F^2 - \|W^{(h+k)}(t)\|_F^2 \right].
$$
This means that if the network is deep  (high $k$) and $p > 1$ then the layer $h+k$ moves exponentially slower than the layer $h$. If $p < 1$, the situation is reversed. Both of these cases can potentially cause trouble. Figure~\ref{fig:phomGF} shows the results of an experiment. The plots confirms this conservation law.

\paragraph{Experiments.} The setup is the same as in Section~\ref{sec:1homog} apart from the architecture used. We use a fully connected NN with with the following dimensions $d_x = d_0 = 784, d_1 = 200, d_2 = 200, d_3 = 200, d_4 = 200, d_y = d_5 = 10$. We plot quantities for $W^{(2)}$ and $W^{(4)}$. Activation functions are $\sigma_1 \equiv \text{ReLU},\sigma_2,\sigma_3 \equiv \text{RePU}(2), \sigma_4 \equiv \text{ReLU}$. Thus the factor $p^k$ that appears in the legend of Figure~\ref{fig:phomog} becomes $2^2 = 4$.

\subsection{Swish}

There are two main variants of the Swish activation function considered in the literature. One fixes the $\beta$ parameter to $1$ for all neurons with the Swish activation, the other makes the $\beta$ trainable. We consider the latter case. More formally, we assume that for some $h \in [K-1]$ the $\sigma_h$ (activations in the $h$-th layer) are Swish functions. Moreover, for every $i \in [d_h]$ the activation function for the $i$-th neuron in layer $h$ is equal to $\text{Swish}\left(\beta^{(h)}_i \right)$, where $\beta^{(h)}_i$ is a trainable parameter. The parameters $\bet^{(h)}$ are trained with the same optimization algorithm as the rest of the network. As for the case of $1$-homogeneous no assumptions are made with respect to the activation functions on any other layer.

\paragraph{Symmetry.} Define $\w_\epsilon$ to be equal to $\w$ apart from:
$$W^{(h)}_\epsilon[i,:] := (1+\epsilon) \cdot W^{(h)}[i,:] ,$$
$$\beta^{(h)}_{\epsilon}[i] := \frac{1}{1+\epsilon} \cdot \beta^{(h)}[i], $$
$$W^{(h+1)}_{\epsilon}[:,i] := \frac{1}{1+\epsilon} \cdot W^{(h+1)}[:,i] \text{.}$$
Then observe that
$$f_{\w} \equiv f_{\w_\epsilon} \text{.}$$
\begin{proof}
Let $a^{(h-1)} \in \R^{d_{h-1}}$ be activation functions of neurons in layer $h-1$. Then
\begin{align*}
&W^{(h+1)}_\epsilon \cdot \text{Swish}(\bet^{(h)}_{\epsilon}) \left(W^{(h)}_\epsilon a^{(h-1)} \right) - W^{(h+1)} \cdot \text{Swish}(\bet^{(h)}) \left(W^{(h)} a^{(h-1)} \right)  \\
&= \text{Swish}(\beta^{(h)}_\epsilon[i]) \left(\left\langle W^{(h)}_\epsilon[i,:], a^{(h-1)} \right\rangle \right) W^{(h+1)}_\epsilon[:,i] - \text{Swish}(\beta^{(h)}[i]) \left(\left\langle W^{(h)}[i,:], a^{(h-1)} \right\rangle \right) W^{(h+1)}[:,i] \\
&= \text{Swish} \left(\frac{1}{1+\epsilon} \cdot  \beta^{(h)}[i] \right) \left(\left\langle (1+
\epsilon) \cdot W^{(h)}[i,:], a^{(h-1)} \right\rangle \right) \frac{1}{1+\epsilon} \cdot W^{(h+1)}[:,i] + \\
&- \text{Swish}(\beta^{(h)}[i]) \left(\left\langle W^{(h)}[i,:], a^{(h-1)} \right\rangle \right) W^{(h+1)}[:,i] \\
&= 0,
\end{align*}
where in the last equality we used the property from Observation~\ref{obs:homogeneity}.
\end{proof}
The derivation of conserved quantities follows the same pattern as the derivation for the $1$-homogeneous and $p$-homogeneous cases. The results are summarized in Table~\ref{tab:homogeneity}.

\paragraph{Implications.} The Swish function was first proposed in \citet{swish}. The authors used an automated search procedure to discover new activation functions. The one that performed the best across many challenging tasks was Swish. It consistently worked better (or on par) with the standard ReLU activation. For the GF the conserved quantity is
$$
\|W^{(h)}\|_F^2 - \|W^{(h+1)}\|_F^2 - \|\beta^{(h)}\|^2_2 = \text{const}.
$$
We leave it as a open question to determine whether the empirically observed performance advantage can be explained in terms of the above conservation law. We also point the reader to Section~\ref{sec:biases} where different conserved quantities are derived and which potentially might lead to an explanation for the advantage.

\subsection{$p$-homogeneous \& Swish with Bias}\label{sec:biases}

\begin{table}
\centering
\begin{tabular}{SS} \toprule
    { }                           & {General Dynamics}    \\ \midrule \rule{0pt}{1ex}
    {(Leaky)ReLU}  & {$\begin{array} {lcl} \left\langle W^{(h)}, \cnstd \dot{W}^{(h)} + \cnstdd \ddot{W}^{(h)} \right\rangle + \left\langle b^{(h)}, \cnstd \dot{b}^{(h)} + \cnstdd \ddot{b}^{(h)} \right\rangle + &  &  \\ - \left\langle W^{(h+1)}, \cnstd \dot{W}^{(h+1)} + \cnstdd \ddot{W}^{(h+1)} \right\rangle & = & 0 \end{array}$}  \\ \rule{0pt}{6ex} 
    {$x^p$ / RePU($p$)}           & {$\begin{array} {lcl} \left\langle W^{(h)}, \cnstd \dot{W}^{(h)} + \cnstdd \ddot{W}^{(h)} \right\rangle + \left\langle b^{(h)}, \cnstd \dot{b}^{(h)} + \cnstdd \ddot{b}^{(h)} \right\rangle + &  &  \\ - p \cdot \left\langle W^{(h+1)}, \cnstd \dot{W}^{(h+1)} + \cnstdd \ddot{W}^{(h+1)} \right\rangle & = & 0 \end{array}$}  \\ \rule{0pt}{6ex} 
    {Swish}  & {$\begin{array} {lcl} \left\langle W^{(h)}, \cnstd \dot{W}^{(h)} + \cnstdd \ddot{W}^{(h)} \right\rangle + \left\langle b^{(h)}, \cnstd \dot{b}^{(h)} + \cnstdd \ddot{b}^{(h)} \right\rangle + &  &  \\ - \left\langle W^{(h+1)}, \cnstd \dot{W}^{(h+1)} + \cnstdd \ddot{W}^{(h+1)} \right\rangle - \left\langle \beta^{(h)}, \cnstd \dot{\beta}^{(h)} + \cnstdd \ddot{\beta}^{(h)} \right\rangle & = & 0 \end{array}$}   \\ 
    \bottomrule
\end{tabular}

\caption{\label{tab:homogeneitybias} Conserved quantities due to the homogeneity of the activation functions with bias terms. Note that we only write it for General Dynamics as all specific algorithms we consider are special cases.}

\end{table}

In this subsection we explore how the presence of biases in the network influences symmetries and conserved quantities. The symmetries presented in the previous subsections in the presence of biases generalize in a natural. For instance for the case of $1$-homogeneous activation the symmetry informally says: multiply weights and biases of layer $h$ by a constant and divide weights of layer $h+1$ by the same constant. More formally if for $h \in [K-1]$ we have that $\sigma_h$ is $p$-homogeneous then the symmetry is:

\paragraph{Symmetry for $p$-homogeneous with bias.} Define $\w_\epsilon = (W^{(1)}_\epsilon, b^{(1)}, \dots, W^{(K)}_\epsilon, b^{(K)})$ to be equal to $\w$ apart from:
$$W^{(h)}_\epsilon[i,:] := (1+\epsilon) \cdot W^{(h)}[i,:] ,$$
$$b^{(h)}_{\epsilon}[i] := (1+\epsilon) \cdot b^{(h)}[i], $$
$$W^{(h+1)}_{\epsilon}[:,i] := \frac{1}{(1+\epsilon)^p} \cdot W^{(h+1)}[:,i] \text{.}$$
Then observe that
$$f_{\w} \equiv f_{\w_\epsilon} \text{.}$$

If for $h \in [K-1]$ we have that $\sigma_h$ is Swish then the symmetry is:

\paragraph{Symmetry for Swish with bias.} Define $\w_\epsilon = (W^{(1)}_\epsilon, b^{(1)}, \dots, W^{(K)}_\epsilon, b^{(K)})$ to be equal to $\w$ apart from:
$$W^{(h)}_\epsilon[i,:] := (1+\epsilon) \cdot W^{(h)}[i,:] ,$$
$$b^{(h)}_{\epsilon}[i] := (1+\epsilon) \cdot b^{(h)}[i], $$
$$\beta^{(h)}_{\epsilon}[i] := \frac{1}{1+\epsilon} \cdot \beta^{(h)}[i], $$
$$W^{(h+1)}_{\epsilon}[:,i] := \frac{1}{1+\epsilon} \cdot W^{(h+1)}[:,i] \text{.}$$
Then observe that
$$f_{\w} \equiv f_{\w_\epsilon} \text{.}$$

The proofs are analogous to the previous cases. The results are summarized in Table~\ref{tab:homogeneitybias}.

\paragraph{Implications.} Observe that conserved quantity for GF can (as before) be written as a total derivative and for $1$-homogeneous functions gives
\begin{equation}\label{eq:biasrelu}
\|W^{(h)}\|_F^2 + \|b^{(h)}\|^2_2 - \|W^{(h+1)}\|_F^2  = \text{const},
\end{equation}
while for Swish it gives
\begin{equation}\label{eq:biasswish}
\|W^{(h)}\|_F^2 + \|b^{(h)}\|^2_2 - \|W^{(h+1)}\|_F^2 - \|\beta^{(h)}\|^2_2 = \text{const}.
\end{equation}
If we assume that all activation functions in the NN are ReLU and sum \eqref{eq:biasrelu} across all layers we will get
$$
\|W^{(1)}\|_F^2 - \|W^{(K)}\|_F^2 + \sum_{h=1}^{K-1} \|b^{(h)}\|^2_2  = \text{const}.
$$
If instead all activation functions are Swish then if we sum \eqref{eq:biasswish} across all layers we will get
$$
\left(\|W^{(1)}\|_F^2 - \|W^{(K)}\|_F^2 \right) + \left( \sum_{h=1}^{K-1} \|b^{(h)}\|^2_2 - \sum_{h=1}^{K-1} \|\beta^{(h)}\|^2_2 \right)  = \text{const}.
$$
We leave it as an open problem to determine if the two above conservation laws can explain the empirical advantage of Swish over ReLU.

\section{Linear Activation Functions Lead To Balancedness of Weight Matrices}

Let us now investigate the invariances for the case of linear activation functions. Not suprisingly, due to the significantly larger group of symmetries present in this case (compared to $p$-homogeneous activations) we arrive at significantly stronger conclusions.

\begin{table}
\centering
\begin{tabular}{SS} \toprule
    {Algorithm }                           & {Conserved Quantities}    \\ \midrule \rule{0pt}{1ex}
    {General Dynamics}  & {$ W^{(h)}\left(\cnstdd \ddot{W}^{(h)} + \cnstd \dot{W}^{(h)}\right)^T - \left(\cnstdd \ddot{W}^{(h+1)} + \cnstd \dot{W}^{(h+1)}\right)^T W^{(h+1)} = 0  $}  \\ \rule{0pt}{3ex} 
    {Newtonian Dynamics}  & {$W^{(h)}\left(\ddot{W}^{(h)}\right)^T - \left(\ddot{W}^{(h+1)}\right)^T W^{(h+1)} = 0 $}  \\ \rule{0pt}{3ex} 
    {Nesterov's Accelerated Gradient Flow}           & {$W^{(h)}\left(\ddot{W}^{(h)} + \frac{3}{t}\dot{W}^{(h)}\right)^T - \left(\ddot{W}^{(h+1)} + \frac{3}{t}\dot{W}^{(h+1)}\right)^T W^{(h+1)} = 0    $}   \\ \rule{0pt}{3ex} 
    {Gradient Flow}  & {$W^{(h)}\left(W^{(h)}\right)^T - \left(W^{(h+1)}\right)^T W^{(h+1)} = \text{const}.$}  \\ 
    \bottomrule
\end{tabular}
\caption{\label{tab:linear}Conserved quantities due to linear activation symmetry.}
\end{table}

\paragraph{Symmetry.} Let $h \in [K-1], \epsilon > 0$, and assume that $\sigma_h$ (activation in the $h$-th layer) is linear, that is $x \mapsto c \cdot x$ for a fixed $a$. Moreover, for a matrix $A \in \text{Mat}(d_h)$ define to be equal to $\w$ apart from:
$$W^{(h)}_\epsilon := (I+\epsilon \cdot A) \cdot W^{(h)},$$
$$W^{(h+1)}_{\epsilon} := W^{(h+1)} \cdot (I + \epsilon \cdot A)^{-1} \text{.}$$
Note that the eigenvalues of $(I + \epsilon \cdot A)$ are of the form $1+\epsilon \lambda_i$, where the $\lambda_i$ are the eigenvalues of the matrix $A$. Therefore, all eigenvalues of $(I + \epsilon \cdot A)$ are strictly positive if $\epsilon$ is sufficiently small. Since we are interested in the case $\epsilon \rightarrow 0$, the inverse $(I + \epsilon \cdot A)^{-1}$ exists. For $f$ as defined in \eqref{eq:defoff} we have
$$f_{\w} \equiv f_{\w_\epsilon} \text{.}$$
\begin{proof}
Let $a^{(h-1)} \in \R^{d_{h-1}}$ be the activation of neurons in layer $h-1$. Then
\begin{align*}
&W^{(h+1)}_\epsilon \sigma_h \left(W^{(h)}_\epsilon a^{(h-1)} \right) - W^{(h+1)} \sigma_h \left(W^{(h)} a^{(h-1)} \right)  \\
&=  W^{(h+1)} (I + \epsilon \cdot A)^{-1}  \sigma_h \left( (I+
\epsilon \cdot A)  W^{(h)} a^{(h-1)}  \right) - W^{(h+1)} \sigma_h \left( W^{(h)} a^{(h-1)}  \right) \\
&= 0,
\end{align*}
where in the last equality we used the fact that $\sigma_h(x) = c \cdot x$.
\end{proof}

\paragraph{Generator.} Now observe that the generator of this symmetry can be associated with a linear transformation $\dx_A : \R^m \rightarrow \R^m$ (which can also be thought of as a matrix but it is simpler to specify the transformation). It is defined as:
\begin{equation}\label{eq:generatorlinear}
\dx_A \left( \left(W^{(1)}, \dots, W^{(K)}\right) \right) := \left( 0, \dots, 0, A \cdot W^{(h)}, - W^{(h+1)} \cdot A, 0, \dots, 0 \right).
\end{equation}
%\begin{align}
%&\dx|_{W^{(h)} \times W^{(h)}} := A \nonumber \\
%&\dx|_{W^{(h+1)} \times W^{(h+1)}} := -A. \label{eq:generatorlinear}
%\end{align}
To see that this is in fact the case let us expand out $W^{(h)}_\epsilon$ and $W^{(h+1)}_{\epsilon}$ to the first order in $\epsilon$. We get
$$W^{(h)}_\epsilon =  W^{(h)} + \epsilon \cdot A \cdot W^{(h)},$$
$$W^{(h+1)}_{\epsilon} = W^{(h+1)}(I + \epsilon \cdot A)^{-1} \approx W^{(h+1)} - \epsilon \cdot W^{(h+1)} \cdot A \text{.}$$
We see that the linear part is equal to $\epsilon \cdot \dx_A \cdot \w$. 

For derivation of generators in terms of Lie theory we refer the reader to Appendix~\ref{apd:lielinear}.

\paragraph{Conserved Quantity.} With the generators in hand we are ready to find conserved quantities for systems evolving according to \eqref{equ:extensioneulerlagrange}. For every $A \in \text{Mat}(d_h)$, the corresponding generator $\dx$ is defined according to \eqref{eq:generatorlinear}. 
Applying \eqref{equ:extensionconserved} for the dynamic \eqref{equ:ode} we get
\begin{align}
0 
&=\langle E(t, \w, \wdot, \wddot), \dx(\w) \rangle \nonumber \\
&= \langle \cnstdd \wddot + \cnstd \wdot, \dx(\w) \rangle \nonumber \\
&= \langle \cnstdd \ddot{W}^{(h)} + \cnstd \dot{W}^{(h)}, A \cdot W^{(h)} \rangle - \langle \cnstdd \ddot{W}^{(h+1)} + \cnstd \dot{W}^{(h+1)} , W^{(h+1)} \cdot A \rangle \nonumber \\
&= \trace\left(\left(\cnstdd \ddot{W}^{(h)} + \cnstd \dot{W}^{(h)}\right)^T A W^{(h)}\right) - \trace\left(\left(\cnstdd \ddot{W}^{(h+1)} + \cnstd \dot{W}^{(h+1)}\right)^T W^{(h+1)} A \right) \nonumber \\
&= \trace\left(\left[ W^{(h)}\left(\cnstdd \ddot{W}^{(h)} + \cnstd \dot{W}^{(h)}\right)^T - \left(\cnstdd \ddot{W}^{(h+1)} + \cnstd \dot{W}^{(h+1)}\right)^T W^{(h+1)}\right] A\right). \label{equ:linearconserved}
\end{align}
To extract from this the desired conserved quantity we use the following lemma. 
\begin{lemma}\label{lem:tracetozero}
For $n \in \N$ let $X \in \text{Mat}(n)$. If 
\begin{align*}
    \trace (X A) = 0,
\end{align*}
for every $A \in \text{Mat}(n)$ then $X = 0$.
\end{lemma}
\begin{proof}
Using the SVD decomposition we can write $X$ as $X = U D V^T$, where $U,V \in U(n)$ (group of unitary matrices) and $D$ is diagonal. Hence
\begin{align}
0 = \trace(X A) = 
\trace(U D V^T A) =
\trace(U D V^T A U U^T) = \trace(U D B U^T) = \trace(D B),
\end{align}
where $B = V^T A U$ can be any matrix (as the assumption holds for every $A \in \text{Mat}(n)$) and where in the last step we have used the cyclic property of the trace. In particular, picking $B$ to be $D$ we get that:
\begin{align*}
\sum_{i=1}^{n} D[i,i]^2 = 0,
\end{align*}
which implies that $D$ must be the zero matrix.
\end{proof}
Applying Lemma~\ref{lem:tracetozero} to (\ref{equ:linearconserved}) we conclude
\begin{equation}\label{eq:linearconserved}
W^{(h)}\left(\cnstdd \ddot{W}^{(h)} + \cnstd \dot{W}^{(h)}\right)^T - \left(\cnstdd \ddot{W}^{(h+1)} + \cnstd \dot{W}^{(h+1)}\right)^T W^{(h+1)} = 0.    
\end{equation}
The left-hand side of equation \eqref{eq:linearconserved} represents the \say{conserved} quantity for linear activation function. This result is the template for entries in Table~\ref{tab:linear}. We observe that for the case when $\cnstdd = 0, \cnstd = 1$ (which corresponds to the GF) the expression from \eqref{eq:linearconserved} can be written as a total derivative. This yields the conserved quantity
\begin{equation}\label{eq:linearGFconserved}
W^{(h)}\left(W^{(h)}\right)^T - \left(W^{(h+1)}\right)^T W^{(h+1)} = \text{const}.    
\end{equation}

\paragraph{Implications.} The conserved quantities derived using the symmetry of linear activation function can be understood as balancedness conditions on consecutive weight matrices. For the case of the GF the condition $W^{(h)}\left(W^{(h)}\right)^T - \left(W^{(h+1)}\right)^T W^{(h+1)} = \text{const}$ was first proved for the case of fully linear NNs in \citet{aroraImplicit}. The proof was later generalized to the case where only one layer need to be linear in \citet{weihubalance}. The fact that this quantity stays conserved was a crucial building block in \citet{aroralinearnetworks} where the authors showed that GD for linear NNs converges to a global optimum. The fact that the weight matrices $W^{(1)}, \dots, W^{(K)}$ remain balanced (which means that $W^{(h)}\left(W^{(h)}\right)^T \approx \left(W^{(h+1)}\right)^T W^{(h+1)}$) throughout the optimization allows one to derive a closed form expression for the evolution of the end-to-end matrix $W^{(K)} \cdot \dots \cdot W^{(1)}$. Then one is able to show that $\frac{d}{dt} \loss(\w)$ is proportional to the minimal singular value of the end-to-end matrix. If the weight matrices are approximately balanced at initialization and a condition on the distance of the initial end-to-end matrix to the target matrix is satisfied (plus a technical assumption on the minimal number of neurons in the hidden layers) then convergence to a global optimum is guaranteed. 

One can understand the condition $W^{(h)}\left(W^{(h)}\right)^T \approx \left(W^{(h+1)}\right)^T W^{(h+1)}$ as implying alignment of the left and right singular spaces of $W^{(h)}$ and $W^{(h+1)}$. This property allows the simplification of the product $W^{(h+1)}W^{(h)}$, which in turn allows one to derive bounds for the evolution of the end-to-end matrix and then prove that the loss decreases at a particular rate.

Even though for ND we cannot write the \say{conserved} quantity as a total derivative it is still possible to prove an approximate balanced condition. Starting with \eqref{eq:linearconserved} we can get for ND that
\begin{equation}\label{eq:NDsecondderivative}
\frac{d^2}{dt^2} \left[\frac12 W^{(h)}\left(W^{(h)}\right)^T - \frac12 \left(W^{(h+1)}\right)^TW^{(h+1)} \right]
= \dot{W}^{(h)}\left(\dot{W}^{(h)}\right)^T - \left(\dot{W}^{(h+1)}\right)^T\dot{W}^{(h+1)}.
\end{equation}
Similarly to the $1$-homogeneous case, for ND we can see that the trace of the right hand side is related to the kinetic energy of the system
$$
\trace \left(\dot{W}^{(h)}\left(\dot{W}^{(h)}\right)^T - \left(\dot{W}^{(h+1)}\right)^T\dot{W}^{(h+1)} \right) = \|\dot{W}^{(h)}\|_F^2 - \|\dot{W}^{(h+1)}\|_F^2 \leq \|\dot{W}^{(h)}\|_F^2 + \|\dot{W}^{(h+1)}\|_F^2.
$$
Balance condition, $W^{(h)}\left(W^{(h)}\right)^T \approx \left(W^{(h+1)}\right)^TW^{(h+1)} $, is no longer satisified perfectly during training but it evolves in a controlled way.

As we already discussed, the approximate balancedness was a crucial building block in \citet{aroralinearnetworks} for showing that GD converges. We leave it as an interesting open problem to adapt the proof in \citet{aroralinearnetworks} using \eqref{eq:NDsecondderivative} as a tool to provide an approximate balance condition. It is worth mentioning that ND will most likely not converge to the optimal solution but oscillate around it. Therefore, the first challenge is to find a proper phrasing of such a potential \say{converge} result. The case of NAGF seems even more challenging.

\section{Data Symmetry Leads To Invariance of Weights}\label{sec:datasymmetries}

In this section we explore how symmetries in data influence the optimization algorithms. We focus on data that is rotational symmetric but the results hold more broadly.

\begin{table}
\centering
\begin{tabular}{SS} \toprule
    { }                           & {Conserved Quantities}    \\ \midrule \rule{0pt}{1ex}
    {General Dynamics}  & {$\left( \cnstdd \ddot{W}^{(1)} + \cnstd \ddot{W}^{(1)} \right)^T \cdot W^{(1)} - \left( W^{(1)} \right)^T \cdot \left( \cnstdd \ddot{W}^{(1)} + \cnstd \ddot{W}^{(1)} \right) = 0$}  \\ \rule{0pt}{3ex}
    {Newtonian Dynamics}  & {$(\ddot{W}^{(1)})^T W^{(1)} - (W^{(1)})^T \ddot{W}^{(1)} = 0$}  \\ \rule{0pt}{3ex} 
    {Nesterov's Accelerated Gradient Flow}           & {$(\ddot{W}^{(1)} + \frac{3}{t} \dot{W}^{(1)})^T W^{(1)} - (W^{(1)})^T (\ddot{W}^{(1)} + \frac{3}{t} \dot{W}^{(1)}) = 0  $}   \\ \rule{0pt}{3ex} 
    {Gradient Flow}  & {$(\dot{W}^{(1)})^T W^{(1)} - (W^{(1)})^T \dot{W}^{(1)} = 0$}  \\ 
    \bottomrule
\end{tabular}

\caption{\label{tab:rotational}Conserved quantities due to rotational symmetry of data.}
\end{table}

Note that symmetries in training datasets if not already present are often enforced by data augmentation. For instance for the image classification tasks it's common to augment the data by adding \say{all} possible translations or rotations of the training images. This makes the datasets invariant to these transformations.

\paragraph{Symmetry.} Assume that the training dataset $S = \left\{\left(\x^{(i)},\y^{(i)}\right)\right\}_{i = 1}^n \subset \R^{d_x} \times \R^{d_y}$ is rotational symmetric. More formally, assume that for every $Q \in \text{SO}(d_x)$ (recall that it is the group of rotations of $\R^{d_x}$) we have that $Q S = S$. Observe that the loss function $\loss$ maps datasets and predictions to real numbers, that is it is of the form $\loss(\w,S) \in \R$. 
%for the loss functions that are expressed in the form $\sum_{i=1}^n \ell(y_i, \pred_{\w}(\x_i))$ the following symmetry holds. 
Now let $Q \in \text{SO}(d_x)$ and define $X : \R^m \rightarrow \R^m$ as the following mapping of the space of parameters:
$$
X \left( \left(W^{(1)}, \dots, W^{(K)} \right) \right) := \left(W^{(1)} \cdot Q, W^{(2)}, \dots, W^{(K)} \right).
$$
Then we have
\begin{align*}
\loss(\w, S)
&= \loss(X(\w), Q^{-1} S) && \text{As $W^{(1)} Q Q^{-1} x = W^{(1)} \x$ for every $\x$} \\
&= \loss(X(\w), S) &&\text{As $S$ is $Q$-symmetric},
\end{align*}
which means that $\loss$ is invariant under transformation $X$.

%\begin{align*}
%\loss(\w)
%&= \sum_{i=1}^n \ell(y_i, h_{\w}(\x_i)) \\
%&= \sum_{i=1}^n \ell(y_i, h_{X(\w)}(Q^{-1}\x_i)) && \text{As $W^{(1)} Q Q^{-1} x = W^{(1)} x$ for every $x$} \\
%&= \sum_{i=1}^n \ell(y_i, h_{X(\w)}(\x_i)) %&&\text{As $S$ is $Q$-symmetric} \\
%&= \loss(X(\w)).
%\end{align*}

\paragraph{Generator.} To find the generators corresponding to these transformations we resort to Lie theory as it is no longer this simple to do in the \say{pedestrian} way. It was easy in the $2$ dimensional case (see Example~\ref{exm:so2generator}) but becomes more involved in higher dimensions. For readers not interested in the derivation the collection of generators corresponding to rotations is a collection of linear transformations of $\R^m$ $\{\dx_{P} \ | \ P \in \text{Mat}(d_x), P + P^T = 0 \}$, where 
\begin{equation}\label{eq:generatorrotations}
\dx_{P} \left( \left(W^{(1)}, \dots, W^{(K)} \right) \right) := \left(W^{(1)} \cdot P, 0, \dots, 0 \right) .
\end{equation}

For derivation of generators in terms of Lie theory we refer the reader to Appendix~\ref{apd:lierotational}.

\paragraph{Conserved Quantity.}
To simplify the presentation we first give invariances for the case where $d_x = 3$ and the NN has only one layer and the output is a scalar (a.k.a. linear regression $\langle \w,\x \rangle$). Only after this introduction will we discuss the general result.

\subsection{Linear regression for ND}\label{sec:linearregression}

We consider a model $\langle \w, \x \rangle$, where $\w,\x \in \R^3$. Considering three generators of the $\mathfrak{so}(3)$ algebra (recall that it is an algebra of skew-symmetric matrices) 
$$\dx_1 = \begin{pmatrix}
0 & 0 & 0\\
0 & 0 & 1 \\
0 & -1 & 0
\end{pmatrix}, \dx_2 = \begin{pmatrix}
0 & 0 & -1\\
0 & 0 & 0 \\
1 & 0 & 0
\end{pmatrix}, \dx_3 = \begin{pmatrix}
0 & 1 & 0\\
-1 & 0 & 0 \\
0 & 0 & 0
\end{pmatrix},$$
we get the following equations 
\begin{align*}
0 
&= \langle E(t, \w, \wdot, \wddot), \dx_i(\w) \rangle \\
&= \langle \cnstdd\wddot + \cnstd\wdot, \dx_i(\w) \rangle,
\end{align*}
which written in coordinates give
\begin{align*}
(\cnstdd \ddot{w}_2 + \cnstd \dot{w}_2) w_3 - (\cnstdd \ddot{w}_3 + \cnstd \dot{w}_3) w_2 = 0 \\
(\cnstdd \ddot{w}_3 + \cnstd \dot{w}_3) w_1 -
(\cnstdd \ddot{w}_1 + \cnstd \dot{w}_1) w_3 = 0
\\
(\cnstdd \ddot{w}_1 + \cnstd \dot{w}_1) w_2 - (\cnstdd \ddot{w}_2 + \cnstd \dot{w}_2) w_1 = 0 
\end{align*}
If we denote by $\pmb{v}$ the angular momentum of $\w$ around the origin we can rewrite above to get the following equation
\begin{equation}\label{eq:rotationallinear}
\cnstdd \dot{\pmb{v}} + \cnstd \pmb{v} = 0.
\end{equation}

%and integrating the above. After integration we get exactly the equations for preservation of angular momentum around the three axis $x,y,z$.
%\begin{align*}
%\dot{w}_2 w_3 - \dot{w}_3 w_2 = \text{const} \\
%\dot{w}_3 w_1 - \dot{w}_1 w_3 = 
%\text{const} \\
%\dot{w}_1 w_2 - \dot{w}_2 w_1 = \text{const}
%\end{align*}

\begin{figure}
\centering
\begin{minipage}{.35\textwidth}
  \centering
  \includegraphics[width=.7\linewidth]{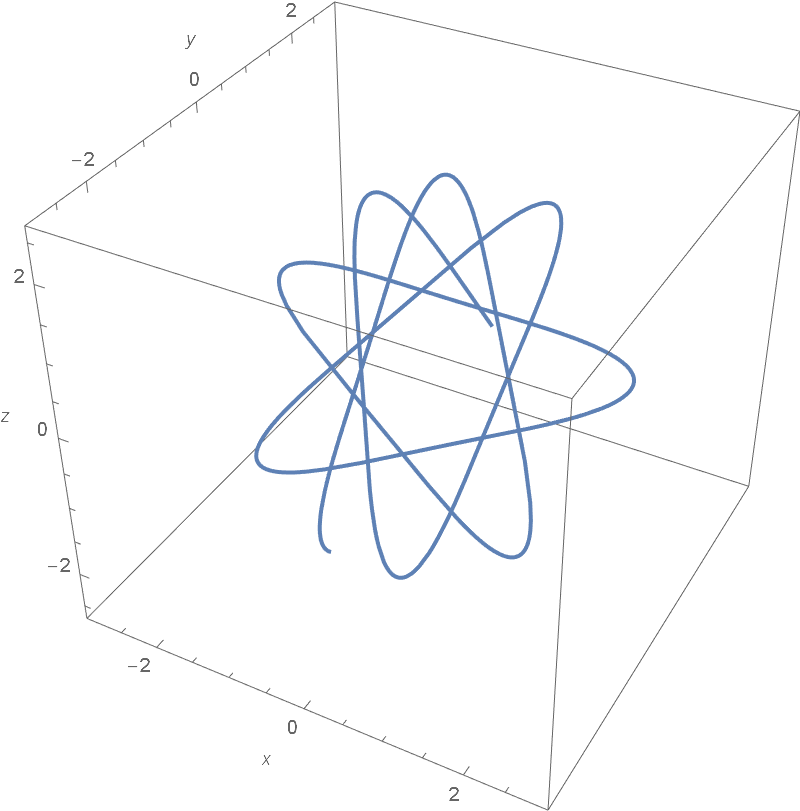}
  %\caption{A figure}
\end{minipage}%
\begin{minipage}{.35\textwidth}
  \centering
  \includegraphics[width=.8\linewidth]{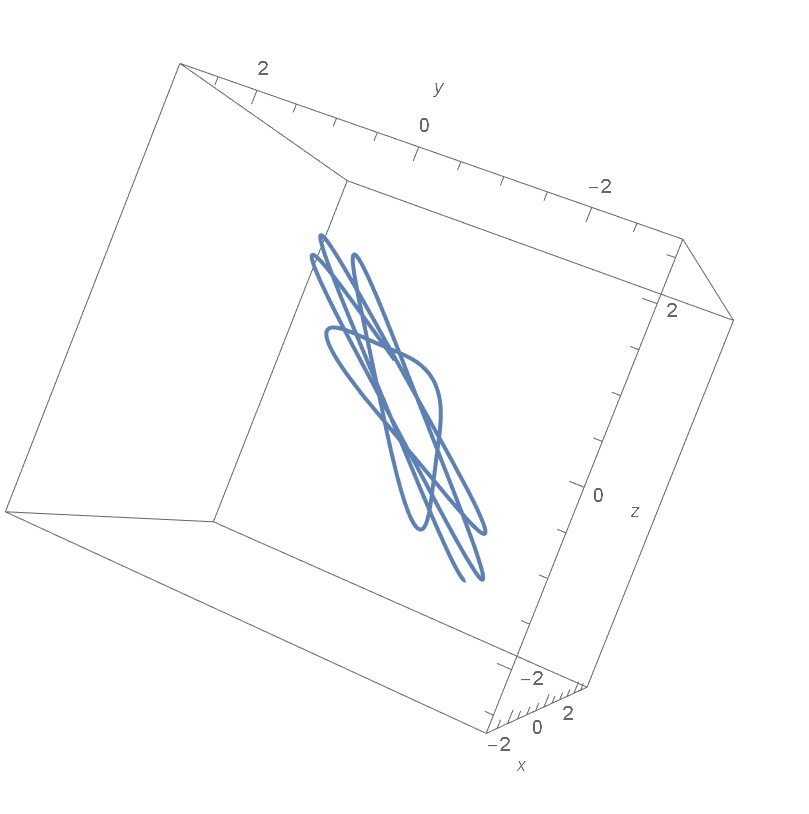}
  %\caption{Another figure}
\end{minipage}
\caption{Trajectory of ND for rotational symmetric distribution.}
\label{fig:NDrotational}
\end{figure}

\paragraph{Implications.} Recall that ND corresponds to the choice $\cnstdd = 1, \cnstd = 0$. Equation (\ref{eq:rotationallinear}) then implies that $\dot{\pmb{v}}=0$, i.e., if we run ND on data that is rotational symmetric then the angular momentum stays preserved. This means that if we initialize the system with $\wdot(0) = 0$ (which implies $0$ angular momentum) then during training only the radial component is changed, i.e.,
$$
\w(t) = \alpha(t) \w(0)
$$
for some function $\alpha : \R \rightarrow \R$. If instead we initialize $\wdot(0) \neq 0$ and we track $\w$ radially projected onto the unit sphere then it will \say{rotate} around the origin with motion that is determined purely by the initialization. Similarly, only the radial component is really influenced by learning.

If we use the parameters $\cnstdd = 0, \cnstd = 1$, we see that the angular momentum is zero for GF. For NAGF the corresponding equation is
\begin{equation}\label{eq:rotationalconservedNAGF}
\dot{\pmb{v}} + \frac{3}{t} \pmb{v} = 0.
\end{equation}
We observe that \eqref{eq:rotationallinear} guarantees that the direction of $\pmb{v}$ stays preserved in general and only the magnitude changes. This guarantees that trajectories of all optimization algorithm we consider are contained in a plane orthogonal to $\pmb{v}$.

\begin{comment}
\begin{align*}
\left(\ddot{w}_2 + \frac{3}{t} \dot{w}_2\right) w_3 - \left(\ddot{w}_3 + \frac{3}{t} \dot{w}_3\right) w_2 = 0, \\
\left(\ddot{w}_3 + \frac{3}{t} \dot{w}_3\right) w_1 - \left(\ddot{w}_1 + \frac{3}{t} \dot{w}_1\right) w_3  = 0, \\
\left(\ddot{w}_1 + \frac{3}{t} \dot{w}_1\right) w_2 - \left(\ddot{w}_2 + \frac{3}{t} \dot{w}_2\right) w_1 = 0.
\end{align*}
\end{comment}

\paragraph{Experiments.} We include plots of experiments on artificial data in Figures~\ref{fig:NDrotational}, \ref{fig:NAGFrotational}, \ref{fig:GFrotational} and \ref{fig:lossrotational}. Figure~\ref{fig:NDrotational} and \ref{fig:NAGFrotational} both depict one trajectory from different angles. The loss to be optimized was $\loss(\w) = (\|\w\|^2 - 1)^2$. The initialization was as follows: for GF $\w = (0.1,0.1,0.1)^T$ and for ND and NAGF $\w = (0.5,0.5,0.5)^T, \wdot = (-2,1,1)^T$ \footnote{Formally, for NAGF we cannot initialize $\wdot$ to a non-zero vector at $t = 0$ as \eqref{eq:rotationalconservedNAGF} implies that $\lim_{t \rightarrow 0}\|\pmb{v}(t)\|^2 = \infty$. In experiments we initialize $\wdot(\delta) = (-2,1,1)^T$ for a small positive $\delta$. In general the initialization for NAGF should be $\wdot(0) = 0$.}. We can see that the trajectory for GF is a simple line segment. This is consistent with angular momentum being zero. The loss for GF (Figure~\ref{fig:lossrotational}, right) is strictly decreasing as expected. 

For ND we can see that the trajectory \say{rotates} around the origin in a plane defined by the equation $y - z = 0$. This is consistent with the conserved angular momentum law as the angular momentum at initialization is equal to $(0,3/2,-3/2)$. This vector is normal to the plane of the orbit of the trajectory which is exactly the plane defined by the equation $y = z$. It is interesting to note the cyclical behavior of the loss function (Figure~\ref{fig:lossrotational}, left), which is related to the fact that the Hamiltonian $\frac12 \|\wdot\|^2 + \loss(\w)$ stays conserved. 

For NAGF we also see that the trajectory stays in the plane defined by the equation $y - z = 0$. We also see \say{rotation} around the origin but in contrast to ND we see dampening both in the trajectory, which converges to a point on the unit sphere $\|\w\|^2 = 1$ and in the loss function (Figure~\ref{fig:lossrotational}, middle).

\begin{figure}
\centering
\begin{minipage}{.35\textwidth}
  \centering
  \includegraphics[width=.7\linewidth]{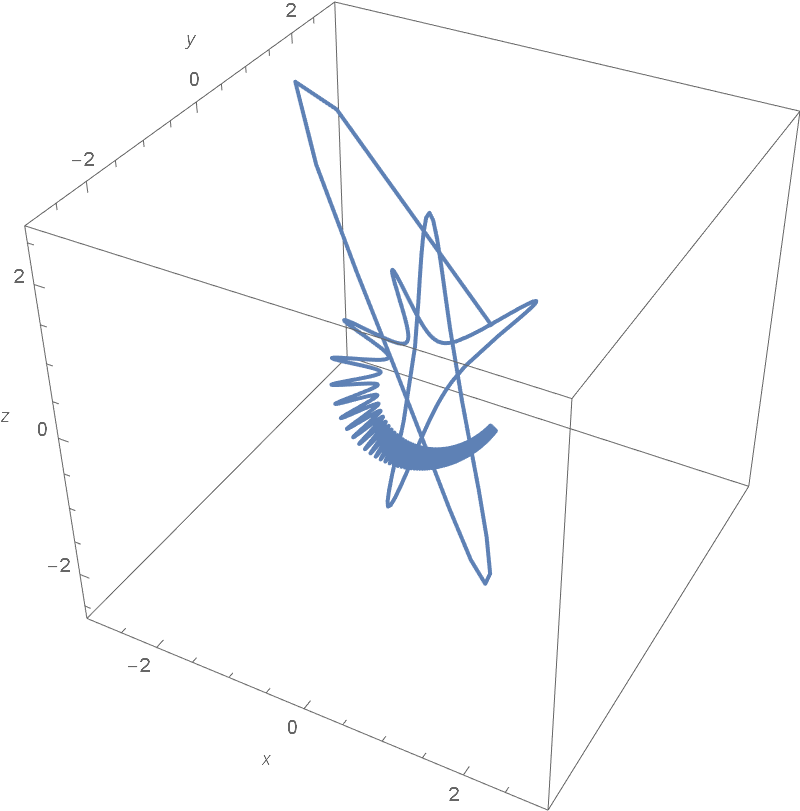}
  %\caption{A figure}
  \label{fig:test1}
\end{minipage}%
\begin{minipage}{.35\textwidth}
  \centering
  \includegraphics[width=.8\linewidth]{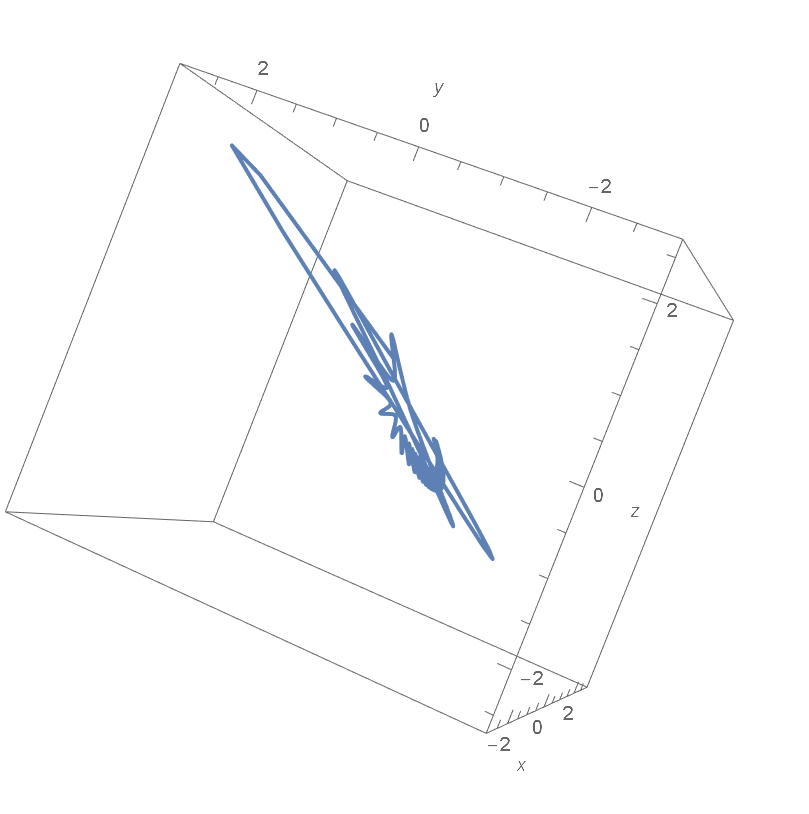}
  %\caption{Another figure}
  \label{fig:test2}
\end{minipage}
\caption{Trajectory of NAGF for rotational symmetric distribution.}
\label{fig:NAGFrotational}
\end{figure}

\begin{figure}
\centering
\begin{minipage}{.35\textwidth}
  \centering
  \includegraphics[width=.7\linewidth]{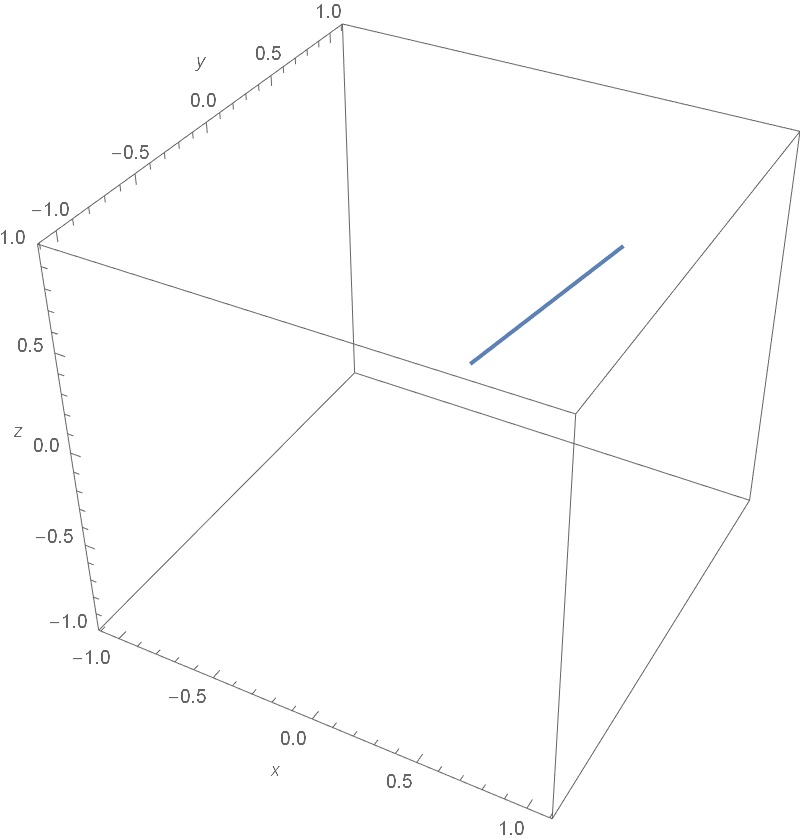}
\end{minipage}%
\caption{Trajectory of GF for rotational symmetric distribution.}
\label{fig:GFrotational}
\end{figure}

\begin{comment}
\begin{figure}
\centering
\begin{minipage}{.33\textwidth}
  \centering
  \includegraphics[width=.99\linewidth]{newtonianloss.png}
  \caption{Loss for ND}
  \label{fig:NDloss}
\end{minipage}%
\begin{minipage}{.33\textwidth}
  \centering
  \includegraphics[width=.99\linewidth]{nesterovloss.png}
  \caption{Loss for NAGF}
  \label{fig:NAGFloss}
\end{minipage}
\begin{minipage}{.33\textwidth}
  \centering
  \includegraphics[width=.99\linewidth]{gradientdescentloss.png}
  \caption{Loss for GF}
  \label{fig:GFloss}
\end{minipage}
\end{figure}
\end{comment}

\begin{figure}[h]
\centering

 \subfloat[ND]{\includegraphics[width=0.33\columnwidth]{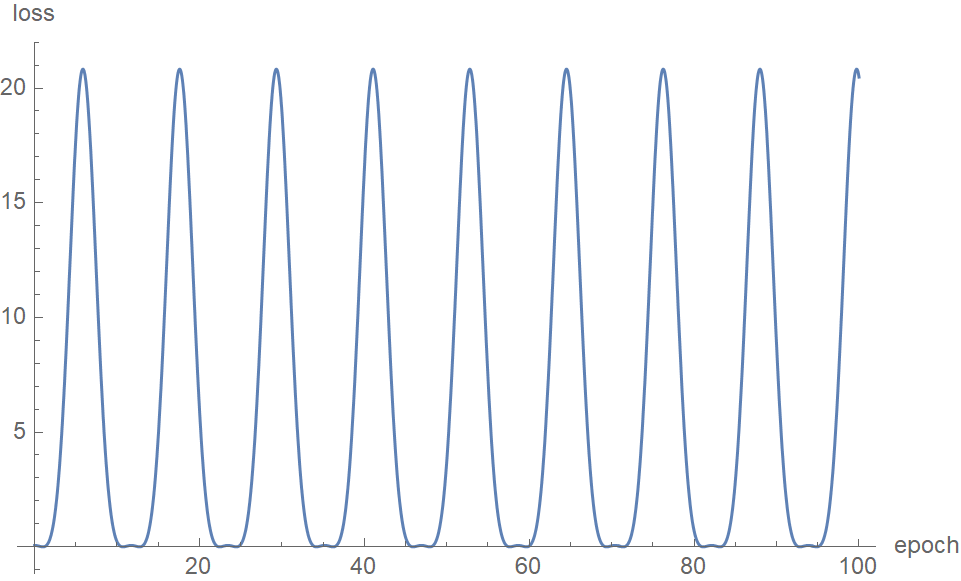}}
 \subfloat[NAGF]{\includegraphics[width=0.33\columnwidth]{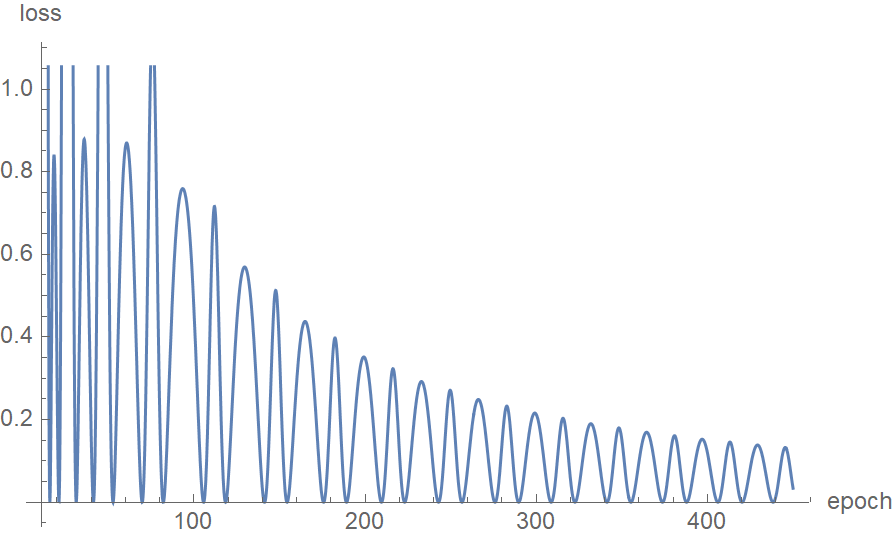}}
 \subfloat[GF]{\includegraphics[width=0.33\columnwidth]{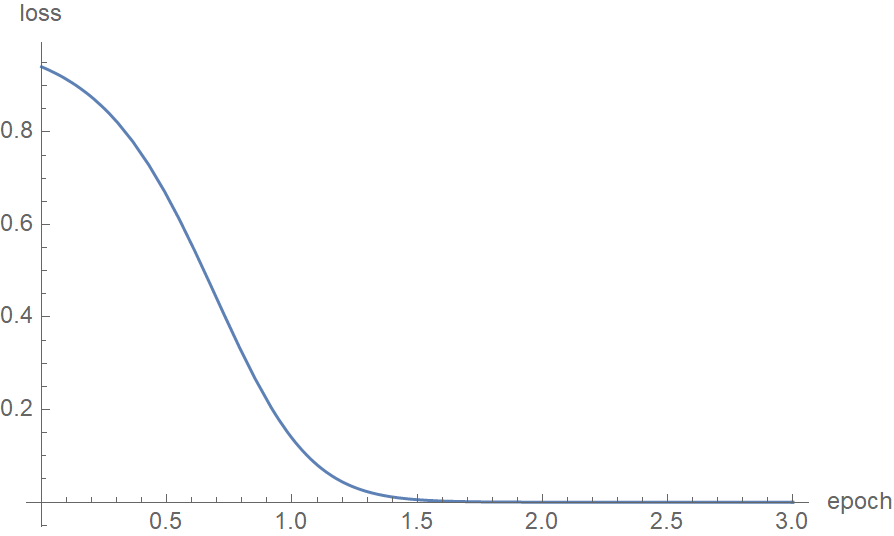}}
 
 \caption{Loss for rotational symmetric data for different optimization algorithms.}\label{fig:lossrotational}
\end{figure}

\subsection{Neural networks}\label{sec:rotationalneural}

After considering this toy example we proceed to discuss the general case.

For every $P \in \mathfrak{so}(d_x)$ applying \eqref{equ:extensionconserved} for the dynamic \eqref{equ:ode} where the generator $\dx$ is defined according to \eqref{eq:generatorrotations} we get

\begin{align*}
0 
&=\langle E(t, \w, \wdot, \wddot), \dx(\w) \rangle \\
&= \langle \cnstdd \wddot + \cnstd \wdot, \dx(\w) \rangle \\
&= \langle \cnstdd \ddot{W}^{(1)} + \cnstd \dot{W}^{(1)}, W^{(1)} \cdot P \rangle \\
&= \trace\left(P \cdot \left( \cnstdd \ddot{W}^{(1)} + \cnstd \dot{W}^{(1)} \right)^T \cdot W^{(1)}\right). 
\end{align*}

%$$\langle \ddot{W}, W \cdot G \rangle = 0 \text{.}$$
%This is equivalent to:
%$$\trace(G \cdot \ddot{W}^T \cdot W) = 0.$$

%{\color{red} 
%I think what you get for $n = 3$ is the following. If you think that the rows of $W$ are positions of $d_y$ particles in $\R^{d_x} = \R^3$ then the angular momentum of the system of these $d_y$ particles is preserved. For higher dimensions of $d_x$ you get high dimensional like angular momentum quantity. I think this equivalent to say that $\ddot{W} W$ is symmetric. It also means that you can integrate it back nicely.
%}
The following lemma helps us to get rid of the trace in the previous expression.
\begin{lemma}\label{lem:rotationtrace}
Let $n \in \N$ and $X \in \text{Mat}(n)$. If 
$$
\trace(P X) = 0
$$
for every $P \in \mathfrak{so}(n)$ then $X$ is symmetric.
\end{lemma}

\begin{proof}
For $i,j \in \{1,\dots,n\}, i<j$ let $P_{i,j} \in \text{Mat}(n)$ be defined as
$$
P_{i,j}[k,l] := \begin{cases}
        1 & \text{if } (k,l) = (i,j)\\
        -1 & \text{if } (k,l) = (j,i) \\
        0 &\text{otherwise}
        \end{cases}. 
$$
Let $i,j \in \{1,\dots,n\}, i<j$. Observe that $P_{i,j} \in \mathfrak{so}(n)$ as it is skew-symmetric. Thus by assumption of the lemma $\trace(P_{i,j} X)$, which implies that $X[j,i] - X[i,j] = 0$. This proves that $X = X^T$. 
\end{proof}
then we arrive at the following conserved quantity
\begin{equation}\label{eq:cnsrvdrotaitonal}
\left( \cnstdd \ddot{W}^{(1)} + \cnstd \dot{W}^{(1)} \right)^T \cdot W^{(1)} - \left( W^{(1)} \right)^T \cdot \left( \cnstdd \ddot{W}^{(1)} + \cnstd \dot{W}^{(1)} \right) = 0
\end{equation}
which is equivalent to saying that $\left( \cnstdd \ddot{W}^{(1)} + \cnstd \dot{W}^{(1)} \right)^T \cdot W^{(1)}$ is symmetric. This result is the template for entries in Table~\ref{tab:rotational}. It is worth noting that \eqref{eq:cnsrvdrotaitonal} captures all relations we can deduce from the symmetry because trace of a product of symmetric and skew-symmetric matrices is zero.

\paragraph{Implications.} For ND (which corresponds to $\cnstdd = 1, \cnstd = 0$) we can write \eqref{eq:cnsrvdrotaitonal} as a total derivative and get
\begin{equation}\label{eq:rotationalND}
\left( \dot{W}^{(1)} \right)^T \cdot W^{(1)} - \left( W^{(1)} \right)^T \cdot \dot{W}^{(1)} = \text{const}.
\end{equation}
This can be interpreted as a generalization for conservation of angular momentum from Section~\ref{sec:linearregression}. For instance if $W^{(1)} \in \text{Mat}(d_1,3)$ and we treat rows $W^{(1)}[i,:]$ as positions of particles in $\R^3$ then \eqref{eq:rotationalND} translates to conservation of angular momentum for the whole system of these $d_1$ particles. When $d_x > 3$ the conservation law becomes the conservation of the higher dimensional generalization of angular momentum. We leave it as an open problem to analyze how it influences the optimization for the higher layers.

It is an interesting research direction to analyze how different symmetries influence optimization. As we mentioned, for image classification tasks it is common to augment the training data by all translations and/or rotations. In \citet{cntk} the following interesting equivalence was shown. Convolutional Neural Tangent Kernels (CNTKs) with a Global Average Pooling (GAP) layer are equivalent to CNTKs without GAP but with data augmentation (all translations). A further work that  explores how symmetry of data influences optimization is \citet{montanariMeanField}. In there the authors consider two layer NN in the infinite width limit. In this limit one can think of the weights of the network as a distribution. They then show that if the data is rotationally symmetric then this distribution on weights is also rotationally symmetric. We leave it as an open problem to cast these phenomena in our framework.

\section{Time Symmetry Leads To Neural Tangent Kernel}

\paragraph{Symmetry.} Note that the Lagrangian for ND does not depend on time.

\paragraph{Conserved quantity.} We are therefore dealing with the special case of Noether's Conservation Law in Lemma~\ref{lem:noether}, where the Hamiltonian $\Lag - \langle \xdot, \frac{\partial \Lag}{\partial \xdot} \rangle$ stays conserved. This implies that $\frac12 \|\wdot\|^2 + \loss(\w)$ is a conserved quantity for ND. 
%As explained in Section~\ref{sec:learning} GF doesn't have a corresponding Lagrangian but can be seen as a \say{massless} strong friction limit. One can use this to derive a similar result as for the case of ND. 

\paragraph{Implications.} Let us now show how this conservation law easily leads to a key property for proving a version of Neural Tangent Kernel (NTK). 

\textit{Informally:} In the sequel, let us refer to the conserved quantity $\frac12 \|\wdot\|^2 + \loss(\w)$ as an \say{energy}. No mathematical property is implied by this language. Since $\frac12 \|\wdot(t=0)\|^2=0$ (this is how we initialize ND), the initial energy of the system is $\loss(\w(0))$ and at most that much energy can potentially be transferred to $\wdot$. As the width of the network tends to infinity, i.e., the number of parameters tends to infinity, the average velocity of the parameters hence decreases. This implies that as the width increases the weights move less and less as their velocities tend to zero during training. This is a key property for showing NTK-like result. 

\textit{More formally:} Noether's Theorem tells us that the energy $\frac{1}{2} \|\wdot \|^2 + \loss(\w)$ stays constant during training. Thus if we initialize $\wdot(0) = 0$ we can conclude that
$$
\|\wdot \|^2 = \sum_{i=1}^m |\dot{w}_i|^2 \leq 2 \loss(\w(0)).
$$ 
From the inequality between arithmetic and quadratic means we further have
\begin{equation}\label{eq:avgvelocitysmall}
\frac{1}{m} \sum_{i=1}^m |\dot{w}_i| \leq \sqrt{ \frac{2 \loss(\w(0))}{m} } \leq O \left(\frac{1}{\sqrt{m}} \right).    
\end{equation}
Note that we can consider $\loss(\w(0))$ as a constant independent of the network width. This is true for popular initialization schemes as e.g., the LeCun initialization. In \citet{chizat19lazy} it is suggested to use \say{unbiased} initializations such that $f_{\w(0)} \equiv 0$. This is often achieved by replicating every unit with two copies of opposite signs. Such initialization would guarantee that $\loss(\w(0))$ is fixed as well.

From \eqref{eq:avgvelocitysmall} we deduce that the average velocity decreases with $m$. In the proofs of NTK a more detailed version of the above is often needed.
\begin{equation}\label{assum:perweightbound}
\max_{i \in [m]} |\dot{w}_i| \leq O \left( \frac{1}{\sqrt{m}} \right).
\end{equation}
This is a key component in the NTK result that allows to bound movement of the weights. Our result doesn't readily provide this more detailed inequality. But it is important to note that our inequality holds with only minor assumptions on the distribution of weights (think of \say{unbiased} initializations) and that we require no assumptions on the architecture of the network whereas in the NTK literature it is typically assumed that weights are initialized with normal distributions. We leave it as an open question whether there are other symmetries, potentially stemming from i.i.d. initializations of the weights or permutation symmetry that could lead to a variant of \eqref{assum:perweightbound}. 
If so, then the approach via symmetries can be viewed as a semi-automatic way of obtaining a key ingredient for NTK-like results. Moreover it woudl hold for more general distributions of weights. What follows is a sketch of a proof of NTK for ReLU networks with scalar outputs trained with ND using the quadratic loss.

\textit{Proof Sketch.}
Integrating \eqref{assum:perweightbound} over time\footnote{We treat learning time as a constant as was also done in \citet{jackotntk}.} we get:
\begin{equation}\label{eq:weightsdontmovemuch}
\max_{i \in [m]}|w_i(T) - w_i(0)| \leq O \left(\frac{1}{\sqrt{m}} \right) \text{.}
\end{equation}
%If the above held for all weights not in expectation then it's kind of easy to show that $\nabla g(\textbf{w}(t)) \approx \nabla g(\textbf{w}(0))$. The reason is that for a netowrk of the form $\sigma^L(W^L \sigma^{L-1}(W^{L-1} \dots))$ the gradient of $g$ with respect to $W^k$ is expressed basically as a product
%$$ 
%(\sigma^{k-1})' \cdot (W^k)^T \cdot (\sigma^{k})' \cdot (W^{k+1})^T \cdot \dots \cdot (W^L)^T \cdot (\sigma^{L})' \cdot \nabla \ell,
%$$
%which means that it doesn't change much if $W^1,W^2, \dots, W^L$ don't change much. To get the result not only in expectation it looks to me that some permutation like symmetry could be used to say that velocities inside each layer cannot concentrate on few weights. 

%Then using this fact they somehow show that $H$ stays constant. For instance for 2-layer NN with ReLU it easily follows that $H$ doesn't change much once we know that weights don't change much (see the talk).
Let us consider the quadratic loss
$$
\loss(\w) = \frac{1}{2} \sum_{i=1}^n (\y^{(i)} - f_{\w}(\x^{(i)}) )^2. 
$$
Recall that we use the ND dynamic \eqref{eq:odeND}. Hence the evolution of weights is governed by
$$
\wddot = \frac{d^2 \w(t)}{d t^2} = - \nabla \loss(\w(t)) = - \sum_{i=1}^n (\y^{(i)} - f_{\w(t)}(\x^{(i)})) \cdot \frac{\partial f_{\w(t)}(\x^{(i)})}{\partial \w}.
$$
We can use this to see how the predictions for inputs $\x^{(i)}$ change during training,
\begin{align}
\frac{d^2 f_{\w(t)}(\x^{(i)})}{dt^2}
&= \frac{d}{dt} \left( \left\langle \frac{\partial f_{\w(t)}(\x^{(i)})}{\partial \w} , \frac{d \w(t)}{dt} \right\rangle \right) \nonumber \\
&= \left\langle \frac{\partial f_{\w(t)}(\x^{(i)})}{\partial \w} , \frac{d^2 \w(t)}{dt^2} \right\rangle + \left\langle \frac{\partial^2 f_{\w(t)}(\x^{(i)})}{\partial \w^2} , \left(\frac{d \w(t)}{dt} \right)^2 \right\rangle \nonumber \\
&= - \sum_{j=1}^n (\y^{(j)} - f_{\w(t)}(\x^{(j)})) \cdot \left\langle \frac{\partial f_{\w(t)}(\x^{(i)})}{\partial \w}, \frac{\partial f_{\w(t)}(\x^{(j)})}{\partial \w} \right\rangle + \left\langle \frac{\partial^2 f_{\w(t)}(\x^{(i)})}{\partial \w^2} , \left(\frac{d \w(t)}{dt} \right)^2 \right\rangle \nonumber \\
&= - \sum_{j=1}^n (\y^{(j)} - f_{\w(t)}(\x^{(j)})) \cdot \left\langle \frac{\partial f_{\w(t)}(\x^{(i)})}{\partial \w}, \frac{\partial f_{\w(t)}(\x^{(j)})}{\partial \w} \right\rangle, \label{eq:derivation}
\end{align}
where in the last transition we used the fact that for ReLU NNs
$$
\frac{\partial^2 f_{\w(t)}(\x^{(i)})}{\partial \w^2} = 0.
$$
Here we have crucially used the fact that the second derivative of $ x \mapsto \max(0,x)$ is $0$. 
%Writing \eqref{eq:derivation} again we get
%$$ 
%\frac{d^2 f_{\w(t)}(\x^{(i)})}{dt^2}
%= \sum_{j=1}^n (f_{\w(t)}(\x^{(j)}) - \y^{(j)}) \cdot \left\langle \frac{\partial f_{\w(t)}(\x^{(i)})}{\partial \w}, \frac{\partial f_{\w(t)}(\x^{(j)})}{\partial \w} \right\rangle
%$$
Using \eqref{eq:weightsdontmovemuch} and i.i.d. initializations of weights one can argue that
$$
\left\langle \frac{\partial f_{\w(0)}(\x^{(i)})}{\partial \w}, \frac{\partial f_{\w(0)}(\x^{(j)})}{\partial \w} \right\rangle \approx \left\langle \frac{\partial f_{\w(t)}(\x^{(i)})}{\partial \w}, \frac{\partial f_{\w(t)}(\x^{(j)})}{\partial \w} \right\rangle.
$$
Denoting $H_{i,j} := \left\langle \frac{\partial f_{\w(0)}(\x^{(i)})}{\partial \w}, \frac{\partial f_{\w(0)}(\x^{(j)})}{\partial \w} \right\rangle$ we can thus write
$$
\ddot{f}_{\w(t)}(\x^{(i)}) \approx - \sum_{j = 1}^n (\y^{(j)} - f_{\w(t)}(\x^{(j)})) \cdot H_{i,j} \text{.}
$$
and denoting $\uu(t) := \left(f_{\w(t)}(\x^{(1)}), f_{\w(t)}(\x^{(2)}), \dots, f_{\w(t)}(\x^{(n)}) \right)^T$ we get
\begin{equation}\label{eq:NDntk}
\ddot{\uu} \approx - H \cdot (\y - \uu).
\end{equation}
%which has almost the same form like \eqref{eq:standardNTK} but with the second derivative. This second order differential equation looks like it produces solutions that will "oscillate around a solution $u(t) \equiv y$" like what happens in one dimensional case. For instance one dimension equation $\ddot{u} = -(u - 3)$ has solutions $u(t) = c_1 \sin(t) + c_2 \cos(t) + 3$.  

We explain now how \eqref{eq:NDntk} corresponds to kernel methods. First we explain how a different dynamic $\dot{\uu} = - H \cdot (\y - \uu)$ corresponds to the kernel regression with GF. Let $\Phi$ be the mapping\footnote{For this proof sketch we assume for simplicity that $\Phi$ maps to a finite dimensional space.} such that $\langle \Phi(\x), \Phi(\x') \rangle$ is the kernel between $\x,\x'$. Then kernel regression can be phrased as learning a linear model in features $\Phi$. That is, the model is $\langle \alphap , \Phi(\x) \rangle$, where $\alphap$ are the weights to be learned. If we consider the square loss we get
$$ 
\loss(\alphap) = \frac{1}{2} \sum_{i=1}^n (\y^{(i)} - \langle \alphap, \Phi(\x^{(i)}) \rangle)^2 .
$$
Now if we optimize $\loss(\alphap)$ using GF the evolution of $\alphap$'s is governed by
$$
\dot{\alpha}_j = - \frac{\partial \loss}{\partial \alpha_j} = - \sum_{i=1}^n (\y^{(i)} - \langle \alphap, \Phi(\x^{(i)}) \rangle) \cdot \Phi(\x^{(i)})_j.
$$
Observing that $\uu(t) = (\langle \alphap(t), \Phi(\x^{(1)}) \rangle, \dots, \langle \alphap(t), \Phi(\x^{(n)}) \rangle)^T$, we obtain
$$
\dot{\uu} = - H \cdot (\y - \uu),
$$
where, as before $H_{i,j} = \langle \Phi(\x^{(i)}), \Phi(\x^{(j)}) \rangle$. Now if we consider the same loss function but we train using the ND dynamic:
$$
\ddot{\alpha}_j = - \frac{\partial \loss}{\partial \alpha_j} = - \sum_{i=1}^n (\y^{(i)} - \langle \alphap, \Phi(\x^{(i)}) \rangle) \cdot \Phi(\x^{(i)})_j
$$
we arrive at \eqref{eq:NDntk}.
%One transition in this "proof" was not necessarily correct. It was when I said that for \textbf{every} weight $w_i$ we have $w_i(t) - w_i(0)$ is small but maybe it's easily fixable.

This means that NNs with ReLU activation functions trained with ND are equivalent (in the infinite width limit) to kernel methods trained with ND.

\hfill$\square$

\begin{remark}
Observe that even though the Lagrangian for NAGF \eqref{eq:odeNAGF} explicitly depends on time the bound from \eqref{eq:avgvelocitysmall} still holds for NAGF. To see that one needs only to recall \eqref{eq:nesterovenergydecreases}, which guarantees that the energy of the system decreases.  

It is also possible to get a similar result for GF. As explained in Section~\ref{sec:learning} GF doesn't have a corresponding Lagrangian but can be seen as a \say{massless} strong friction limit. One can use this to derive the following bound
$$
\frac{1}{m} \sum_{i=1}^m |w_i(T) - w_i(0)| \leq O \left( \frac{1}{\sqrt{m}} \right).
$$
\end{remark}

\begin{comment}
\section{Experimental Setup}

First three plots. $ W_4 = (10,200), W_3 = (200,200), W_2 = (200,200), W_1 = (200,784) $. Measuring $W_3,W_2$. All ReLU activation functions.

$p$-homogeneous plot.  $ W_5 = (10,200), W_4 = (200,200), W_3 = (200,200), W_2 = (200,200), W_1 = (200,784) $. Measuring $W_4,W_2$. In between W4-W3 and W3-W2 are RePU(2). Other ReLU.

For rotational symmetric. The loss function is $(x^2 + y^2 + z^2 - 1)^2$. Starting point in gradient flow is $(0.1,0.1,0.1)$, in the rest $(0.5,0.5,0.5)$. Starting velocity is $(-2,1,1)$.
\end{comment}

\section{Conclusions}
In the literature on NNs symmetries appear frequently and in crucial roles. Their use however was mostly implicit and sometimes appeared coincidental. We introduce a framework in order to analyze symmetries in a systematic way. We show how a chosen optimization algorithm interacts with the symmetric structure, giving rise to conserved quantities. These conservation laws factor out some degrees of freedom in the model, leading to more control on the path taken by the algorithm.

The conserved quantities we obtain often have the form of balance-like conditions. E.g., for the case of homogeneous activation functions this allows one to bound Frobenius norms of weights matrices. This bound effectively guarantees that weights remain bounded during training. This boundedness is a crucial property for showing convergence of gradient based methods. For linear activation functions symmetry leads to a balance condition on consecutive layers that allows to analytically simplify the expression for the end-to-end matrix. This in turn makes it possible to guarantee progress in every GD step. Data augmentation naturally leads to a symmetry that in turn restricts trajectories of optimization algorithms. We left it as an open problem to analyze implications of this restriction. Lastly the time invariance of the optimization algorithm  leads to a bound on the average movement of weights during training. As discussed, this is intimately connected to NTK-like results and potentially can lead to extensions, shedding new light on the connection between NNs and kernel methods.

There are many directions for future work. 

\paragraph{Further implications.} First, it will be interesting to see what further implications can be derived from the conserved quantities discussed in the paper. For instance, it might be possible to prove convergence results for NAGF for deep linear networks. 

\paragraph{Discreet symmetries.} The symmetries we consider in this work are continuous. A natural question is if it is possible to extend the framework to discrete symmetries (e.g. permutation symmetry of neurons). The physics literature has not much to offer on this question. 

\paragraph{Generalization bounds.} Ultimately what we care about is the performance of the model on real distributions. Do conservation laws play any role for the question of generalization?

\paragraph{Discreet algorithms.} Our analytical results hold for continuous dynamics (ODEs) but the algorithms used in practice are discrete. We demonstrated in the experiments that the derived quantities stay approximately conserved also in the discrete setting. It would be useful to elucidate the robustness of the results under discretization.
%It is possible that our result are pretty robust to discretization.

\paragraph{GD vs SGD.} The relationship between symmetries and the optimization algorithm are clearest if we consider GD. This is the point of view we took. But in practice the learning algorithm is typically SGD. To what degree do the conservation laws still hold in this case? As a simple preliminary result we note that conserved quantities we derived due to symmetries of the activation functions also hold for SGD (to the same level as they carry over from GF to GD). This is because the symmetries leave not only the loss invariant but also the prediction function itself.  It is an interesting question whether similar results hold more broadly.

\section{Acknowledgements}
We would like to thank Nicolas Macris for introducing us to the wonderful world of Noether's Theorem.

\bibliography{references}

\appendix

\section{From Invariances to Generators - Lie Groups and Lie Algebras}\label{apd:lie}

Recall that in general (see Definition~\ref{def:generatorsandinvariance}) we are interested in transformations of time and space
\begin{align*}
&t \mapsto t' = T(t, \x, \epsilon), \\
&\x \mapsto \x' = X(t, \x, \epsilon).
\end{align*}
As explained in Section~\ref{sec:extensions}, for the most part we focus on transformations that leave only the potential $\pot(\x)$ invariant. This means that transformations simplify to
$$
\x \mapsto \x' = X(\x,\epsilon).
$$
Assume that we have found a transformation of the form $\dx(t,\x) = \dx(\x)$ (a transformation that does not depend explicitly on time). Think of $\dx(\x)$ as a vector field $\dx : \R^d \xrightarrow{} \R^d$, i.e., $\dx(\x)$ associates to every point $\x$ in \say{space} the vector $\dx(\x)$. Then for every point $\x$ in space, $X(\x,\epsilon)$ defines a curve that is parametrized by $\epsilon$, and this curve can be described in terms of a differential equation involving $\dx$:
\begin{align} 
&X(\x,0) = \x, \nonumber \\
& \frac{d}{d\epsilon} X(\x,\epsilon) = \dx(X(\x,\epsilon)). \label{equ:integralcurve}
\end{align}
As mentioned earlier we would like to approximate $X$ to the first order by $\dx$,
\begin{equation}\label{eq:linearize}
X(\x, \epsilon) = \x + \dx(\x) \epsilon + O(\epsilon^2).
\end{equation}
A mathematical theory that analyzes continuous symmetries and their relations with generators is the one of Lie groups -- smooth manifolds with a group structure. The key idea in the theory is to replace the \textit{global} object (for us $X$) by it's \textit{local} or linearized version (for us $\dx$), where the relation between the two is given by \eqref{eq:linearize}. The localized version is known as the Lie algebra or as the infinitesimal group. This theory is exactly what we need to go from symmetries to generators. For an in depth treatment of Lie theory we refer the reader to \citet{knapp2002lie}.

Typical examples of transformations that might leave $\pot(\x)$ invariant are translations, see  e.g., Example~\ref{ex:spatialsymmetry}, scalings,  or rotations. In general, there can be more than one particular transformation for a given setup. E.g., think of rotations by unitary matrices. It is useful to think of the the whole {\em collection} of such transforms.  Each such collection will correspond to a Lie group $G$ that acts on the points $\x$. In other words, $G$ acts on $\R^d$. In more detail, $g \in G$ corresponds to some transformation $X(\cdot, \epsilon)$, for some fixed $\epsilon$, in a smooth way. Let us look at some examples.

\begin{example}[$\text{SO}(2)$ Group]
Our running example for understanding Lie groups and Lie algebras will be the one of $\text{SO}(2)$ - the group of rotations of $\R^2$. Every $g \in \text{SO}(2)$ can be thought of as a matrix 
\begin{align}
\begin{pmatrix} \label{equ:sorotation}
\cos(\theta) & \sin(\theta) \\
-\sin(\theta) & \cos(\theta)
\end{pmatrix}
\end{align}
that acts on $\R^2$ by matrix multiplication  and corresponds to a rotation by an angle $\theta$ around the origin.
\end{example}

Given a transformation, how do we find the corresponding generators $\dx$? To do that one considers a linearization of $G$ around the identity $e$. For every $g \in G$ one thinks of it as a sequence of consecutive transformations by an infinitesimal parameter, that is expresses
\begin{equation}\label{eq:expandaroundidentity}
g = \lim_{\eta \xrightarrow{} 0} (e + \eta \cdot \dx_g)^{1/\eta},
\end{equation}
where $\dx_g$ is the sought for generator associated with $g$. How are the addition and multiplication by $\eta$ in \eqref{eq:expandaroundidentity} defined? For the cases we consider, the involved Lie groups have matrix representations and thus one can think of addition and multiplication as operations on matrices.
The general case is more involved and we refer the reader
to the relevant literature \citep{knapp2002lie}.

\begin{example}[$\text{SO}(2)$ Group Contd]
Consider our running example. We can express any $g \in \text{SO}(2)$ in the form  \eqref{eq:expandaroundidentity} as 
\begin{equation} \label{eq:generatorforrotation}
\begin{pmatrix}
\cos(\theta) & \sin(\theta) \\
-\sin(\theta) & \cos(\theta)
\end{pmatrix} = \lim_{\eta \xrightarrow{} 0} \left[ \begin{pmatrix}
1 & 0 \\
0 & 1
\end{pmatrix} + \eta \cdot \begin{pmatrix}
0 & \theta \\
-\theta & 0
\end{pmatrix} \right]^{1/\eta}.
\end{equation}
This identity comes from the fact that for small angles $\eta \theta$ we have $\cos(\eta \theta) \approx 1, \sin(\eta \theta) \approx \eta \theta$ and thus 
$$
\begin{pmatrix}
\cos(\eta \theta) & \sin(\eta \theta) \\
-\sin(\eta \theta) & \cos(\eta \theta)
\end{pmatrix} \approx \begin{pmatrix}
1 & \eta \theta \\
-\eta \theta & 1
\end{pmatrix}.
$$
A rotation by an angle $\theta$ can therefore be thought as $1/\eta$ rotations by an angle $\eta \theta$. Equation \eqref{eq:generatorforrotation} says that $\begin{pmatrix}
0 & \theta \\
-\theta & 0
\end{pmatrix}$ is the generator corresponding to the rotation (\ref{equ:sorotation}).
%the $\begin{pmatrix}
%\cos(\theta) & \sin(\theta) \\
%-\sin(\theta) & \cos(\theta)
%\end{pmatrix}$. 
\end{example}
Once we found $\dx_g$ corresponding to $g$, Lie theory tells us that a vector field defined as
$$\x  \mapsto \dx_g (\x) $$
corresponds to $g$ in the following sense. For every $\x \in \R^d$ if one starts at $\x$, solves the differential equation (\ref{equ:integralcurve}) and travels along the integral curve of $\dx_g$ for time $1$ then one arrives exactly at the point $g \cdot \x$. 

\begin{example}[$\text{SO}(2)$ Group Contd]\label{exm:so2generator}
We can verify that the vector field corresponding to (\ref{equ:sorotation}) is defined exactly by the action of the corresponding generator $\dx_g$ on $\R^2$
\begin{equation}\label{eq:explicitrot}
\dx_g (\x) = \begin{pmatrix}
0 & \theta \\
-\theta & 0
\end{pmatrix} \cdot \x = \begin{pmatrix}
0 & \theta \\
-\theta & 0
\end{pmatrix} \cdot \begin{pmatrix}
x_1 \\
x_2
\end{pmatrix} = \theta \cdot \begin{pmatrix}
x_2 \\
-x_1
\end{pmatrix}.
\end{equation}
\end{example}
Lie theory guarantees that a collection $\{\dx_g \}_{g \in G}$ can be associated with the Lie algebra $\mathfrak{g}$ of $G$. Formally, Lie algebras are vector spaces with an operation called the \textit{Lie bracket}. The Lie algebra associated with $G$ can be thought of as a tangent space to $G$ at identity, which for our purposes should be understood through \eqref{eq:expandaroundidentity}. 

\begin{example}[\text{SO}(2) Group Contd]
For a group of rotations in $d$ dimensional space $G = \text{SO}(d)$ the associated Lie algebra $\mathfrak{g}$ is a collection of skew-symmetric matrices. Note that this description is consistent with \eqref{eq:generatorforrotation}, where we found that the Lie algebra is equal to $\left\{\begin{pmatrix}
0 & \theta \\
-\theta & 0
\end{pmatrix} : \theta \in \R \right\}$, which, for this simple case, happens to be isomorphic with $\R$. This simplification is no longer true in higher dimensions and is important in our applications in Section~\ref{sec:datasymmetries}. 
\end{example}

Once we have a generator $\dx \in \mathfrak{g}$, Section~\ref{sec:extensions} prescribes to analyze the set of equations \eqref{equ:extensionconserved}
$$\langle E(t, \x, \xdot, \xddot), \dx(\x) \rangle = 0. $$
Note that for every $\dx$ we get a set of $d$ equations as $\x(t) \in \R^d$ but we get many more equations as \eqref{equ:extensionconserved} can be applied to every element of $\mathfrak{g}$. In our applications it will be crucial to analyze how the whole collection of equations
$$
\{\langle E(t,\x,\xdot, \xddot), \dx(\x) \rangle = 0 \ | \ \dx \in \mathfrak{g} \}
$$
gives rise to \say{conserved} quantities. Intuitively speaking, different generators from $\mathfrak{g}$ can lead to different, linearly independent, equations which, combined together, give more information than any single one. One can think that the number of non-redundant sets of equations is \say{equal} to $\dim(\mathfrak{g})$.

\subsection{Summary and Recipe} \label{sec:recipe}
Let us summarize this section. We give a high level recipe for finding \say{conserved} quantities for systems evolving according to \eqref{equ:extensioneulerlagrange}, i.e., according to

$$E(t, \x, \xdot, \xddot) = - \nabla_{\x} \pot(\x). $$

\begin{algorithm}[H] 
\caption{\label{alg:recipe} \textsc{FindConservedQuantities}($E,\pot$) \Comment $E$ is the left-hand side of \eqref{equ:extensioneulerlagrange}
\newline \text{ }\Comment $\pot$ is the potential from \eqref{equ:extensioneulerlagrange}}
\label{alg:random-walk-tt}
\begin{algorithmic}[1]
	\State Find a Lie group $G$ whose action leaves $\pot$ invariant
	\State Let $\mathfrak{g}$ be the Lie algebra associated with $G$ 
	\State $\text{CQ} :=$ Inferred conserved quantities from a collection of equations $\{\langle E(t,\x,\xdot, \xddot), \dx(\x) \rangle = 0 \ | \ \dx \in \mathfrak{g} \}$
	\State \Return $\text{CQ}$
\end{algorithmic}
\end{algorithm}

\section{Finding Generators with Lie Theory}

As we discussed before it is possible to find generators corresponding to continuous symmetries using Lie theory. For readers familiar with Appendix~\ref{apd:lie} we derive now generators for symmetries considered in the paper.

\subsection{Homogeneity of Activation}\label{apd:liehomogeneity}

The transformation can be thought of as a group $\R^{+}$ (positive real numbers with multiplication) acting on $\R^m$ via matrix multiplication. More concretely, for $g \in \R^{+}$ the associated matrix is $\bar{g} \in \text{Mat}(m)$ defined as a matrix with non-zero entries on the diagonal only and
$$
\bar{g}[j,j] := \begin{cases}
        g & \text{if the $j$-th index corresponds to a weight in $W^{(h)}[i,:]$}\\
        \frac{1}{g} & \text{if the $j$-th index corresponds to a weight in $W^{(h+1)}[:,i]$} \\
        1 &\text{otherwise}
        \end{cases}. 
$$
To find the corresponding generator $\dx_g$ Appendix~\ref{apd:lie} prescribes to express $g$ as
$$
\bar{g} = \lim_{\eta \xrightarrow{} 0} (I + \eta \cdot \dx_g)^{1/\eta},
$$
which corresponds to \eqref{eq:expandaroundidentity}. One can verify that $\dx_g = \log(g) \cdot \dx$, where $\dx$ is defined in \eqref{eq:generator1homog}. This guarantees that the generator corresponding to $g$ is $\log(g) \cdot \dx$. Lie theory guarantees that the Lie algebra corresponding to $\R^+$ is $\R$. This manifests itself in the fact that $\{\log(g) \ | \ g \in \R^+ \} = \R$. Finally, the collection of generators for $\R^+$ can be associated with the corresponding collection of matrices acting on $\R^m$ by multiplication
$$
\{a\dx \ | \ a \in \R\}.
$$

\subsection{Linear Activation}\label{apd:lielinear}

The transformation can be thought of as a group $\text{GL}(d_h)$ (invertible matrices with multiplication) acting on $\R^m$. More concretely for $g \in \text{GL}(d_h)$ the associated linear action is defined as
$$
g \cdot \left(W^{(1)}, \dots, W^{(K)}\right) := \left( W^{(1)}, \dots, W^{(h-1)}, g \cdot W^{(h)}, W^{(h+1)} \cdot g^{-1}, W^{(h+2)}, \dots, W^{(K)} \right).
$$
To find the corresponding generator $\dx_g$ Appendix~\ref{apd:lie} prescribes to express action $g$ as
$$
g = \lim_{\eta \xrightarrow{} 0} (e + \eta \cdot \dx_g)^{1/\eta},
$$
which corresponds to \eqref{eq:expandaroundidentity}. One can verify that $\dx_g = \dx_{\log(g)}$, where $\dx_{\log(g)}$ is defined according to \eqref{eq:generatorlinear} and the $\log$ of  matrix is defined as a matrix whose exponent is equal to argument of the $\log$. In general it might happen that there are many possible logarithms of a given matrix. However Lie theory guarantees that the whole collection $\{\log(g) : g \in \text{GL}(d_h) \}$ can be associated with the Lie algebra $\text{Mat}(d_h)$. Finally the collection of generators for $\text{GL}(d_h)$ can be associated with the corresponding collection of linear transformations of $\R^m$
$$
\{\dx_{a} \ | \ a \in \text{Mat}(d_h)\}.
$$

\subsection{Data Symmetry}\label{apd:lierotational}

The transformation can be thought of as a group $\text{SO}(d_x)$ (orthogonal matrices with determinant $1$) acting on $\R^m$. More concretely for $g \in \text{SO}(d_x)$ the associated linear action is defined as
$$
g \cdot \left(W^{(1)}, \dots, W^{(K)}\right) := \left(W^{(1)} \cdot g, W^{(2)}, \dots, W^{(K)} \right) .
$$
To find the corresponding generator $\dx_g$ Appendix~\ref{apd:lie} prescribes to express action $g$ as
$$
g = \lim_{\eta \xrightarrow{} 0} (e + \eta \cdot \dx_g)^{1/\eta},
$$
which corresponds to \eqref{eq:expandaroundidentity}. One can verify that $\dx_g = \dx_{\log(g)}$, where $\dx_{\log(g)}$ is defined according to \eqref{eq:generatorrotations} and the $\log$ of  matrix is defined as a matrix whose exponent is equal to argument of the $\log$. In general it might happen that there are many possible logarithms of a given matrix. However Lie theory guarantees that the whole collection $\{\log(g) : g \in \text{SO}(d_x) \}$ can be associated with the Lie algebra of skew-symmetric matrices $\mathfrak{so}(d_x)$. Finally the collection of generators for $\text{SO}(d_x)$ can be associated with the corresponding collection of linear transformations of $\R^m$
$$
\{\dx_{P} \ | \ P \in \text{Mat}(d_x), P + P^T = 0\}.
$$

\end{document}